\documentclass{article}

\usepackage{hyperref}
\usepackage{footnote}
\usepackage{amssymb}
\usepackage{amsmath,amsthm}
\usepackage{stmaryrd}
\SetSymbolFont{stmry}{bold}{U}{stmry}{m}{n} % to get rid of a stmaryrd warning
\usepackage[noend]{algorithmic}
\usepackage{algorithm,caption}
\usepackage{multicol, paracol, multirow}
\usepackage{color}
\usepackage{float}
\usepackage{todonotes}
\usepackage{bm}
\usepackage[latin1]{inputenc}

\usepackage{thmtools}
\usepackage{thm-restate}

\usepackage{hyperref}
\usepackage{cleveref}

\usepackage{dsfont}

\usepackage{letltxmacro}

\def\isundefined#1{TT\fi\ifx#1\undefined}

\if\isundefined{\stocversion}
\fi

\newlength{\commentindent}
\makeatletter
\renewcommand{\algorithmiccomment}[1]{\unskip\hfill\makebox[\commentindent][l]{//~#1}\par}
\LetLtxMacro{\oldalgorithmic}{\algorithmic}
\renewcommand{\algorithmic}[1][0]{%
	\oldalgorithmic[#1]%
	\renewcommand{\ALC@com}[1]{%
		\ifnum\pdfstrcmp{##1}{default}=0\else\algorithmiccomment{##1}\fi}%
}
\makeatother

\usepackage{amsfonts,epsfig,color,graphicx,verbatim, enumitem}
\usepackage{multirow}

\newif\ifhyper\IfFileExists{hyperref.sty}{\hypertrue}{\hyperfalse}
\hypertrue
\ifhyper\usepackage{hyperref}\fi

\usepackage{enumitem}

\usepackage{framed}
\usepackage{nicefrac}

\def\nnewcolor{1}
\ifnum\nnewcolor=1

\fi
\ifnum\nnewcolor=0

\fi

\def\colorful{0}

\ifnum\colorful=1

\else

\fi

\newtheorem{theorem}{Theorem}[section]

\newtheorem{lemma}[theorem]{Lemma}
\newtheorem{informal theorem}[theorem]{Theorem (informal statement)}

\newtheorem{corollary}[theorem]{Corollary}
\newtheorem{claim}[theorem]{Claim}

\newtheorem{remark}[theorem]{Remark}
\newtheorem{observation}[theorem]{Observation}
\theoremstyle{definition}
\newtheorem{definition}[theorem]{Definition}
\newcommand{\eqdef}{\stackrel{{\mathrm {\scriptstyle def}}}{=}}

\newcommand{\err}{\mathrm{err}}

\newcommand{\R}{\mathbb{R}}

\newcommand{\Z}{\mathbb{Z}}
\newcommand{\N}{\mathbb{N}}

\newcommand{\eps}{\epsilon}

\newcommand{\pr}{\mathop{\mathbf{Pr}}}
\renewcommand{\Pr}{\mathbf{Pr}}
\newcommand{\poly}{\mathrm{poly}}

\newcommand{\sgn}{\mathrm{sign}}
\newcommand{\sign}{\mathrm{sign}}

\DeclareMathOperator*{\E}{\mathbb{E}}

\newcommand{\CCC}{\mathcal{C}}

\newcommand{\data}{\ensuremath{\mathcal{X}}}

\newcommand{\exor}{\mathrm{EX}}

\mathchardef\mhyphen="2D
\newcommand{\supp}{\ensuremath{\text{supp}}}
\newcommand{\X}{\ensuremath{\mathcal{X}}}

\newcommand{\vecs}{{\vec{s}}}
\newcommand{\vecsone}{{\vec{s_1}}}
\newcommand{\vecstwo}{{\vec{s_2}}}

\newcommand{\AAA}{\mathcal{A}}
\newcommand{\M}{\mathcal{M}}

\newcommand{\pmone}{\{\pm1\}}
\newcommand{\Exp}{\mathbb{E}}
\newcommand{\stat}{\mathtt{STAT}}
\newcommand{\rstat}{\mathtt{rSTAT}}

\newcommand{\alphaoff}{\alpha_{\mathrm{off}}}

\newcommand{\Xset}{\data_{\mathrm{set}}}

\newcommand{\rhh}{\mathtt{rHeavyHitters}}
\newcommand{\rhhfull}{{\rhh_{\rho, v, \epsilon}}}
\newcommand{\rhhfullrhoprime}{{\rhh_{\rho', v, \epsilon}}}

\newcommand{\rmedian}{\mathtt{rMedian}}
\newcommand{\median}{\mathtt{Median}}
\newcommand{\rmedofmedians}{\mathtt{rMedianOfMedians}}

\newcommand{\rhalfspace}{\mathtt{rHalfspaceStL}}
\newcommand{\rhalfspacewkl}{\mathtt{rHalfspaceWkL}}

\newcommand{\rhalfspacewklfull}{\rhalfspacewkl(\vec s; r)} %version to mesh better with Toni's definition, but somewhat cluncky for filtering model

\newcommand{\rhalfspacewklboxesfull}{\rhalfspacewkl^{\text{box}}(\vec s; r)}

\newcommand{\rhalfspacewklboxes}{\rhalfspacewkl^{\text{box}}}

\newcommand{\constructfoam}{\mathtt{ConstructFoams}}
\newcommand{\constructboxes}{\mathtt{ConstructBoxes}}

\newcommand{\orth}{\bot}                    
\newcommand{\rwkl}{\ensuremath{\mathtt{rWkL}}}
\newcommand{\hz}{h_{\vec{z}}}

\newcommand{\rboost}{\ensuremath{\mathtt{rBoost}}}
\newcommand{\rboostfull}{\ensuremath{\rboost^{\rwkl}(\vec s; r)}}
\newcommand{\wkl}{\ensuremath{\mathtt{WkL}}}

\newcommand{\boldh}{\mathbf{h}}
\newcommand{\rejectionsampler}{{\mathtt{RejectionSampler}}}
\newcommand{\rejectionsamplerfull}{\rejectionsampler(\vec s_{\text{all}}, m_{\text{target}}, \mu; r)}

\newcommand{\mtarget}{{m_{\text{target}}}}
\newcommand{\skept}{{\vec s_{\text{kept}}}}
\newcommand{\sall}{{\vec s_{\text{all}}}}
\newcommand{\BBB}{\mathcal{B}}

\newcommand{\Acc}{\ensuremath{\mathtt{Acc}}}

\title{Reproducibility in Learning\footnote{Our choice of the term ``reproducibility'' is inconsistent with current guidance from the Association for Computing Machinery regarding usage of the terms ``reproducibility'' and ``replicability'' ~\cite{acm20}, of which we were regrettably unaware at the time of publication. In future work, we adopt more consistent terminology, using ``replicability'' to refer to the same stability notion termed ``reproducibility'' in this work.}}

\author{ 
	Russell Impagliazzo\thanks{Supported by the Simons Foundation and NSF grant CCF-1909634.} \\ 
	University of California San Diego \\ 
	\texttt{russell@eng.ucsd.edu}
	\and 
	Rex Lei\footnotemark[1] \\ 
	University of California San Diego \\ 
	\texttt{rlei@eng.ucsd.edu}
	\and 
	Toniann Pitassi\thanks{Supported by NSERC, the IAS School of Mathematics and NSF grant CCF-1900460} \\
	Columbia University, IAS, University of Toronto \\ \texttt{tonipitassi@gmail.com}
	\and 
	Jessica Sorrell\footnotemark[1] \\ 
	University of California San Diego \\ 
	\texttt{jlsorrel@eng.ucsd.edu}
}
\date{}

\begin{document}

\maketitle

\begin{abstract}
We introduce the notion of a \textit{reproducible} algorithm in the context of learning. A reproducible learning algorithm is resilient to variations in its samples --- with high probability, it returns  the exact same output when run on two samples from the same underlying distribution.
We begin by unpacking the definition, clarifying how randomness is instrumental in balancing accuracy and reproducibility.
We initiate a theory of reproducible algorithms, showing how reproducibility 
implies desirable properties such as data reuse and efficient testability.
Despite the exceedingly strong demand of reproducibility,
there are efficient reproducible algorithms for several fundamental problems in statistics and learning.
First, we show that any statistical query algorithm can be made reproducible
with a modest increase in sample complexity, and we use this to construct
reproducible algorithms for finding approximate heavy-hitters and medians.
Using these ideas, we give the first reproducible algorithm for learning halfspaces via a reproducible weak learner and a reproducible boosting algorithm. Interestingly, we utilize a connection to foams \cite{KDRW12} as a higher-dimension randomized rounding scheme.
Finally, we initiate the study of lower bounds and inherent tradeoffs for reproducible algorithms, giving nearly tight sample complexity upper and lower bounds for reproducible versus nonreproducible SQ algorithms.
% showing that reproducibility sometimes comes with a price.

%We formally define this notion of stability and consider its relationship to similar notions. We show that any statistical query algorithm can be made reproducible, and show a separation by providing a reproducible algorithm that cannot be efficiently simulated in the statistical query model. We show that reproducibility is a strictly stronger notion than differential privacy. Finally, we show that reproducible algorithms allow for adaptive data reuse. 

%We introduce a strong stability notion for data analysis algorithms called \emph{reproducibility}. An algorithm is reproducible if, run on two samples drawn from the same distribution, it yields the same output with high probability. We formally define this notion of stability and consider its relationship to similar notions. We show that any statistical query algorithm can be made reproducible, and show a separation by providing a reproducible algorithm that cannot be efficiently simulated in the statistical query model. We show that reproducibility is a strictly stronger notion than differential privacy. Finally, we show that reproducible algorithms allow for adaptive data reuse. 
\end{abstract}

%\tableofcontents

\setcounter{page}{0}

\thispagestyle{empty}

\newpage 

\section{Introduction}
\label{sec:introduction}

Reproducibility is vital to ensuring scientific conclusions are reliable,
and researchers have an obligation to ensure that their results are replicable.
In the last twenty years, lack of reproducibility has been a major issue in nearly all scientific areas of study. For example, a 2012 Nature article by Begley and Ellis
reported that the biotechnology company Amgen was only able to replicate 6 out of 53 landmark studies in haematology and oncology \cite{Nature:BE12}.
In a 2016 Nature article, Baker published a survey of 1500 researchers, reporting that 70\% of scientists had tried and failed to replicate the findings of another researcher,
and that 52\% believed there is a significant crisis in reproducibility~\cite{baker2016scientists}.

A key issue underlying the reproducibility crisis 
(as articulated in many articles, e.g., \cite{ioannidis})
is the fact that new data/publications
are growing at an exponential rate, giving rise to
an explosion of methods for data generation, screening, testing, and analysis, where, crucially, only the combinations producing the
most significant results are reported.
Such practices (also known as P-hacking, data dredging, and researcher degrees of freedom) can lead to erroneous findings that appear to be significant, but
that don't hold up when other researchers attempt to
replicate them.
Identifying and mitigating these problems is quite subtle. First,
is not easy to come up with an agreed-upon set of practices
that guarantees reproducibility, and secondly,
testing to determine whether or not a finding is statistically significant
is a complex task.

Within the subfields of machine learning and data science,
there are similar concerns about the reliability of published findings.
The performance of models produced by machine learning algorithms may be affected by the values of random seeds or hyperparameters chosen during training, and performance may be brittle to deviations from the values disseminated in published results \cite{henderson2017reinforcement,IHGP17,NIPS:LKMBG18}.
To begin addressing concerns about reproducibility, several prominent machine learning conferences have begun hosting reproducibility workshops and holding reproducibility challenges, to promote best practices and encourage researchers to share the code used to generate their results~\cite{PVSLBdFL20}.

In this work, we aim to initiate the study of reproducibility as a property of algorithms themselves, rather than the process by which their results are collected and reported. We define the following notion of reproducibility, which informally says
that a randomized algorithm is reproducible if two distinct runs of the algorithm on two sets of samples drawn from the same distribution, with internal randomness fixed between both runs, produces the {\it same} output with high probability.

\begin{definition}[Reproducibility]
        \label{def:reproducibility-sample}
        Let $D$ be a distribution over a universe $\data$, and let $\AAA$ be a randomized algorithm with sample access to $D$.
%Let ${\vec s}$ denotes a sequence of samples drawn from $D$, and let $r$ denote a random binary string representing the internal randomness of $\mathcal{A}$.
$\mathcal{A}({\vec s})$ is \textit{$\rho$-reproducible}
        if
        $$\Pr_{{\vec s_1},{\vec s_2},r} \left[
        \mathcal{A}({\vec s_1}; r) = \mathcal{A}({\vec s_2}; r) 
        \right]
        \ge 1 - \rho,$$
        where ${\vec s_1}$ and ${\vec s_2}$ denote sequences of samples drawn i.i.d. from $D$, and $r$ denotes a random binary string representing the internal randomness used by $\mathcal{A}$.
\end{definition}

Our definition of reproducibility
%(Definition~\ref{def:reproducibility-sample})
is inspired by the literature on pseudodeterministic algorithms, particularly the work of Grossman and Liu~\cite{GrossmanLiu:19} and Goldreich~\cite{ECCC:Gol19}. 
%The key difference between this new notion and the one found in the setting of pseudodeterministic algorithms is that 
In the pseudodeterministic setting, the primary concern is reproducing the output of an algorithm given the same input, over different choices of the algorithm's internal randomness.
Our notion (Definition~\ref{def:reproducibility-sample}) is more suitable for the setting of machine learning, where it is desirable to reproduce the exact same output of an algorithm (with high probability) over different sample sets drawn from a distribution $D$.

%\medskip

We observe the following key properties of Definition~\ref{def:reproducibility-sample}.

%\medskip

\textbf{Stability.} 
Reproducibility is a strong stability property that implies independent parties can replicate previous results with high probability, so long as the randomness used to achieve these results is made public. 
For researchers solving machine learning and data analysis tasks, reproducibility allows researchers to verify published results with high probability, as long as the datasets are drawn from the same distribution. 

\textbf{Generalization.}
Reproducibility implies generalization. A reproducible learning algorithm, with high probability, outputs a hypothesis $h$ such that the difference between the risk of $h$ and the empirical risk of $h$ on the training set is small. Intuitively, reproducibilitiy implies that $h$ is independent of the training set with high probability. Thus, a Hoeffding bound can be applied to bound the risk in terms of the empirical risk.

\textbf{Privacy.} 
Differential privacy (DP) is an important notion that requires small distance between the two distributions induced by an algorithm, when run on any two datasets that differ in a single element. Crucially, it asks for the guarantees in the \textit{worst case over datasets}.
%Reproducibility asks for equality between outputs of an algorithm \emph{with high probability}, for a fixed random string.
Reproducible algorithms guarantee a different form of privacy:
If $\mathcal{A}$ is reproducible, then what
$\AAA$ learns (for example, a trained classifier) is almost always the same; thus, $\AAA$
is usually \textit{independent} of the chosen training data. 
In this way,
reproducible algorithms are prevented from
memorizing anything that is specific to the training data,
similar to differentially private algorithms. 
%The basic property of dp algorithms is that the output distribution
%is smooth with respect to the training data: on every neighboring
%pair of training data $s_1,s_2$, the output distributions 
%must be pointwise nearly the same.
Reproducibility is weaker than
differential privacy in the sense that reproducibility only applies to in-distribution
samples, whereas differential privacy applies to \textit{any} training set.
On the other hand, reproducibility is stronger in the sense that its guarantee for in-distribution samples
is global rather than local (for neighboring samples).
%(In fact we prove (Section ?) that $\rho$-reproducibility implies
%$(0,\rho)$-dp.)

\textbf{Testability.} 
While differential privacy has become the standard for privacy-preserving computation, an important issue that is the subject of extensive
research is testing and verifying differential privacy.
%As differential privacy has become the gold standard for privacy,
%the problem of verifying differential privacy
%has become .
As discussed in \cite{GNP20},
DP-algorithms and their implementations are usually analyzed by hand, and
proofs of differential privacy are often intricate and prone to
errors. 
Implementing such an algorithm in practice
often gives rise to DP leaks, due to coding errors or
assumptions made in the proof that do not hold on finite
computers (such as the ability to sample from continuous distributions).
Moreover, the complexity of verifying differential privacy is hard.
Verification in the black-box setting (where the auditor has oracle access to the learning
algorithm) was recently shown to be infeasible, as low
query complexity implies high values of the
the privacy parameters $\epsilon$ and $\delta$ \cite{GilbertM}. In the 
white-box setting where
$\mathcal{A}$ is given to the tester, 
%as a randomized circuit,
%recent results of Gaboardi, Nissim and Purser 
\cite{GNP20} shows that testing for differential privacy is $coNP^{\# P}$-complete.
This has led to an active research area aiming at
developing automated as well as interactive testing and verification methods
for differential privacy
\cite{NarayanFPH15,GaboardiHHNP13,RP10,AH18,BartheGAHRS15,BCKS21,FJ14,ZhangK17}.
In contrast, reproducibility is a form of privacy that can be {\it efficiently}
\textit{tested} in (randomized) polynomial time
(in the dimension of the data universe and $\rho$).

\subsection{Our Main Results}

\subsubsection{Reproducibility: Properties and Alternative Definitions}

We discuss alternative definitions of
reproducibility and show that they are all essentially equivalent.
Then, we also prove some other nice properties
of reproducible algorithms. (All formal statements and proofs
are in Appendix A.)

\begin{enumerate}
\item {\bf Alternative Definitions and Amplification.}
%We show (see the Appendix for formal statements and proofs) that reproducible algorithms satisfy the following properties:
We start by discussing two alternative definitions of
reproducibility and relate them to our definition.
First, we can generalize the definition to include algorithms
${\mathcal A}$ that not only have access to internal randomness and to random samples from an underlying
distribution $D$, but that also have access to
extra non-random inputs. This more general definition captures both
the original definition of pseudodeterministic algorithms as well
as our definition of reproducible learning algorithms, and all of
our results remain unchanged. Second, we discuss an alternative
two-parameter definition, and show that the definitions are qualitatively
the same. We show how to amplify the reproducibility parameter
by a standard argument where the sample complexity is increased modestly.

\item {\bf Public versus Private Randomness.}
Recall that we define reproducibility as the probability that an algorithm returns the same answer
when run twice using different random samples from $D$ but the same internal randomness.
In \cite{GrossmanLiu:19}, the authors define a related concept
in which the internal randomness is divided into two pieces,
public and private randomness, but the
algorithm should return the same answer when just the public
randomness is held fixed. We show that, without loss of generality,
it suffices to use only public randomness.

\item \textbf{Reproducibility Implies Generalization.}
Learning algorithms attempt to use finite samples to generate hypotheses on unknown, possibly complex distributions. 
The error of a hypothesis $h$ on the underlying distribution is called the generalization error. A reproducible algorithm outputs the same hypothesis with high probability, and thus the algorithm seldom draws distinctions between specific samples and the entire distribution.

%\item {\bf Connections to Differential Privacy.} 
%We prove that any $\rho$-reproducible algorithm with sample
%complexity $m$ implies a
%$(0,2\rho)$-differentially private algorithm with sample
%complexity $O(m)$. (IMPROVE)

\item \textbf{Connections to Data Reuse.} 
We explore the connection between reproducible algorithms
and the adaptive data analysis model discussed in
\cite{Dworketal:2014} and \cite{Dworketal:2015}.
We
show that reproducible algorithms are
strongly resilient against adaptive queries.
Informally, with respect to
reproducible algorithms, the sample complexity and accuracy of
(reproducibly) answering $m$ adaptively chosen queries
behaves similarly to the sample complexity and accuracy 
of reproducibly answering $m$ {\it nonadaptively} chosen queries.
\end{enumerate}

\subsubsection{Upper Bounds}
Our main technical results are reproducible algorithms for
some well-studied statistical query and learning problems that are used as building blocks in many other algorithms.

\begin{enumerate}

\item{\bf Simulating SQ Algorithms.} In
Section \ref{sec:statistical-queries}, we give a generic algorithm that reduces the problem of $\rho$-reproducibly estimating a single statistical query with tolerance $\tau$ and error $\delta$ to that of {\it nonreproducibly} estimating the same query within a smaller tolerance and error.

%\begin{restatable}[Simulating SQ algorithms reproducibly]{theorem}{sqalg}
%        Let $A$ be an SQ algorithm that makes $q \in \poly(1/\tau, 1/\delta)$ queries to an SQ oracle with tolerance $\tau$ and failure rate $\delta$. Then there exists a reproducible SQ algorithm $A'$ that makes $q$ queries to an SQ oracle with tolerance $\tau/(\poly(q) + 1)$ and error rate $1/\poly(q)$ that is $\poly(\tau, \delta)$-reproducible.
%\end{restatable}

\begin{theorem}[Theorem \ref{thm:reproducible-sq-oracle}, Restated]
%\begin{restatable}[Simulating Statistical Queries Reproducibly]{theorem}{sqalg}
Let $\psi: \data \rightarrow \{0,1\}$ be a statistical query.
Then the sample complexity of $\rho$-reproducibly estimating
$\psi$ within tolerance $\tau$ and error $\delta$ is at most the
sample complexity of (nonreproducibly) estimating $\psi$ within
tolerance $\tau' = \tau \rho$ and error $\delta' = \tau \delta$.
\end{theorem}

The basic idea is to obtain an estimate of
the statistical query with a smaller tolerance $\tau'$ and then
use a randomized rounding scheme where the interval $[0,1]$
is divided into intervals of size roughly $\tau/\rho$.
Then, every 
value in the interval is rounded
to the midpoint of the region it occurs in.
The partition into intervals is chosen with a random offset 
so that with high probability nearby points will lie in the same region.

\medskip

\item {\bf Heavy-hitters.} Using our simulation of SQ queries, in Section \ref{sec:heavy-hitters}, we demonstrate the usefulness of reproducibility by giving a reproducible algorithm $\rhh$ for identifying approximate $v$-heavy-hitters of a distribution, i.e. the elements in the support of the distribution with probability mass at least $v$.

\begin{lemma}[Lemma \ref{lem:rHeavyHitters-correctness}, Restated]
        %\label{lem:rHeavyHitters-correctness}
        For all $\eps \in (0, 1/2)$, $v \in (\eps, 1- \eps)$, with probability at least $1-\rho$,
        $\rhh_{\rho, v, \eps}$ is $\rho$-reproducible, and returns a list of $v'$-heavy-hitters for some $v' \in [v - \eps, v + \eps]$. Furthermore, the sample complexity is bounded by $\widetilde{O}(\rho^{-2})$.
\end{lemma}

%\begin{restatable}[Reproducible Heavy Hitters]{lemma}
        %%\label{lem:rHeavyHitters-correctness}
        %For all $\eps \in (0, 1/2)$, $v \in (\eps, 1- \eps)$, with probability at least $1-\rho$,
        %$\rhh^\exor (v, \eps)$ is $\rho$-reproducible, and returns a list of $v'$-heavy-hitters for some $v' \in [v - \eps, v + \eps]$. Furthermore, the sample complexity is bounded by $\widetilde{O}(\rho^{-2})$.
%\end{restatable}

The high level idea of our algorithm is to first draw sufficiently
many samples, ${\vec s_1}$, $Q_1=|{\vec s_1}|$, so that with high probability all heavy-hitters
are in ${\vec s_1}$. In the second stage, we draw a fresh set 
${\vec s_2}$ of $Q_2$ many
samples and use them to empirically estimate the 
density of each element in ${\vec s_1}$, and remove those that
aren't above the cutoff $v'$, where $v'$ is
chosen randomly from $[v-\epsilon,v+\epsilon]$ to avoid boundary issues.

\item{\bf Median Finding.} 
In Section \ref{sec:approximate-median}, we design a reproducible algorithm for finding an approximate median in an arbitrary distribution over a finite domain.
Approximate median finding is a fundamental statistical problem,
and is also extensively studied in the privacy literature. 
%Like in the setting of differential privacy,
%reproducible median finding is a key building block for making
%other algorithms reproducible.
%For example, using a known reduction from \cite{BNSV:15},
%a $\rho$-reproducible approximate median algorithm implies
%an algorithm of comparable sample complexity for reproducibly PAC-learning (one-dimensional)
%threshold functions.

%In particular, for any
%problem where the correct answers for an interval, and we have
%a nonreproducible algorithm with correctness probability at least $1/2$,
%we can run our $\rho$-reproducible approximate median algorithm on the
%distributions of outputs of the original algorithm to construct a
%reproducible version.

\begin{theorem}[Theorem \ref{thm:median}, Restated]
                Let $\tau, \rho \in [0,1]$ and let $\delta = 1/3$. Let $D$ be a distribution over $\data$, where $|\data| = 2^d$. Then $\rmedian_{\rho, d, \tau, \delta}$ 
%(Algorithm~\ref{alg:rep-median}) 
is $\rho$-reproducible, outputs a $\tau$-approximate median of $D$ with success probability $1-\delta$, and has sample complexity
        $$ \tilde{\Omega}\left( \left( \frac{1}{\tau^2(\rho - \delta)^2}\right) \cdot \left(\frac{3}{\tau^2}\right)^{\log^{*}|\data|} \right)$$
\end{theorem}

To describe the key ideas in the algorithm, we first
show how approximate-median finding is useful for turning many
algorithms into reproducible ones.
Consider any problem where the correct answers form an interval,
and assume we start with a (not-necessarily) reproducible algorithm that is
mildly accurate. Then we can run a reproducible approximate-median finding algorithm
on the distribution of outputs of the
original algorithm to construct a very accurate reproducible algorithm.

We will actually use this strategy recursively to reproducibly solve approximate median itself. 
Our algorithm recursively composes a mildly accurate reproducible median algorithm with a generic very accurate non-reproducible median algorithm.  
This recursive technique is inspired by, but simpler than,
previous algorithms in the privacy literature \cite{BNSV:15, KLMNS20}, and like these
algorithms,  the sample complexity
of our algorithm has a non-constant but very slowly growing dependence on
the domain size.

\item{\bf Learning Halfspaces.} In Section~\ref{sec:halfspaces}, we obtain a reproducible algorithm $\rhalfspacewkl$ for weakly learning halfspaces.
In Section~\ref{sec:boosting}, we transform it into a reproducible strong learner by way of a reproducible boosting algorithm $\rboost$. 
We stress that our algorithms for halfspaces are reproducible in the
stronger distribution-free setting.

\begin{theorem}[Corollary \ref{cor:reproducible-strong-halfspace-learner}, Restated]
        Let $D$ be a distribution over $\R^d$, and let $f: \R^d \rightarrow \pmone$ be a halfspace with margin $\tau$ in $D$.
        For all $\rho, \eps > 0$.
        Algorithm $\rboost$ run with weak learner $\rhalfspacewkl$
        $\rho$-reproducibly returns a hypothesis
        $\boldh$ such that,
        with probability at least $1-\rho$,
        $\Pr_{\vec{x} \sim D} [\boldh(\vec{x}) = f(\vec{x})]  \ge 1 - \eps$.
        Furthermore, the overall sample complexity is 
%	$O(\frac{\sqrt d}{\tau^2 \rho}^{2.2})$.
	$	\widetilde O 
	\left(
	\frac{d^{10/9}}{\tau^{76/9} \rho^{20/9} \eps^{28/9}}
	\right)
	$. 
\end{theorem}

%Our approximate median algorithm for answering statistical queries can be used
%to reproducibly learn 1-dimensional halfspaces using a known reduction
%from \cite{BNSV:15}.
%At a high level, this algorithm started with an SQ algorithm with
%tight concentration/tolerance, and we then applied a randomized
%rounding scheme in order to argue that the reproducibility of
%the resulting algorithm.
%In order to reproducibly learn higher dimensional thresholds (halfspaces),
%we will follow a similar approach.
In order to reproducibly learn halfspaces, we start with a simple weak learning algorithm for halfspaces 
\cite{Ser02PAC} that takes examples $(\vec{x}_i,y_i) \in \data \times \pmone$, normalizes them, and returns the halfspace defined by vector $\sum_i \vec{x}_i \cdot y_i$.
We show a concentration bound on the sum of normalized vectors from a distribution, and then argue that all vectors within the concentration bound are reasonable hypothesis with non-negligible advantage.

Our randomized rounding scheme is a novel application of the randomized
rounding technique developed in the study of foams \cite{KDRW12}.
The concentration bound together with the foams rounding scheme \cite{KDRW12} yields a reproducible halfspace weak learner.
We then obtain our reproducible strong learner for halfspaces by
combining it with a (new) reproducible boosting algorithm.
Our algorithm is sample efficient but inefficient with
respect to runtime, due to the inefficiency of the foams rounding scheme.
We also give another randomized rounding procedure that gives a polynomial-time strong reproducible halfspace learner,
but with polynomially larger sample complexity.

\end{enumerate}

%\medskip

\subsubsection{The Price of Reproducibility.}
In Section \ref{sec:coin-problem-sq-lb} we ask what is the cost of turning a nonreproducible
algorithm into a reproducible one.
We first show that a $\tau$-tolerant $\rho$-reproducible SQ algorithm $\mathcal{A}$
for $\phi$ implies a $\rho$-reproducible algorithm for the $\tau$-coin problem: given samples from a $p$-biased coin with the promise that either
$p \geq 1/2 + \tau$ or $p \leq 1/2 - \tau$,
determine which is the case.
Our main result in this section are nearly tight
upper and lower bound bounds of $\Theta(\tau^{-2}\rho^{-2})$
on the sample complexity of
$\rho$-reproducibly solving the $\tau$-coin problem 
(for constant $\delta$), and thus the same bounds for 
$\rho$-reproducibly answering SQ queries.
On the other hand, it is well-known that the {\it nonreproducible}
sample complexity of the $\tau$-coin problem is
$\Theta(\tau^{-2} \log(1/\delta))$ (see, e.g. \cite{mousavi}).
So the 
cost of guaranteeing $\rho$-reproducibility for SQ queries is
a factor of $\rho^{-2}$.

For upper bounds, our
generic algorithm in Section \ref{sec:statistical-queries} converts any SQ
query into a reproducible one: if our end goal is a $\rho$-reproducible
algorithm for estimating a statistical query with tolerance $\tau$
and error $\delta$, then the sample complexity is at most
the sample complexity of {\it nonreproducibly} answering the query
to within tolerance $\tau'$, and success probability $1-\delta'$
where $\tau' = O( \tau \rho)$ and $\delta' = O(\delta \tau)$,
which has sample complexity 
$O( \tau^{-2} \rho^{-2} \log(1/\delta'))$.
The main result in this section is the following lower bound
for $\rho$-reproducibly answering statistical queries.

\begin{theorem}[Theorem \ref{thm:sq-reproducibility-lb}, Restated]
        Let $\tau > 0$ and let $\delta < 1/16$.
Any $\rho$-reproducible algorithm for solving the $\tau$-coin coin problem
with success probability at least $1-\delta$ 
requires sample complexity $\Omega(\tau^{-2}\rho^{-2}).$
\end{theorem}

%first show a separation between the sample complexity of reproducible algorithms and statistical query algorithms, via our reproducible heavy-hitter algorithm.
%The algorithm $\rhh$ has sample complexity independent of the size of the domain $\X$, but we show an ensemble of distributions such that any statistical query algorithm must make a number of queries to its oracle that depends on $\X$.
%\begin{restatable}[Learning Heavy-hitters using Statistical Queries]{claim}{sqquery}
%        Any statistical query algorithm for the $v$-heavy-hitters problem requires
%        $\Omega(\log_{1/\tau} |\data|)$
%        calls to the SQ oracle.
%\end{restatable}
%Finally, we state some relationships between reproducibility and differential privacy, and adaptive data anslysis.
%
%Given the similarity between reproducibility and privacy, it is
%natural to ask if we can generically transform a $\rho$-reproducible
%algorithm into an (approximate) differentially private one.
%We partically answer this question by showing that $\rho$-reproducibility implies
%$(0,4\rho)$-differential privacy.

%Lastly, we show that answering statistical queries reproducibly allows for adaptive data reuse without significantly compromising the validity of the analysis.

\paragraph{Related Work.}
A subset of these results \cite{ILS:21:tpdp} was presented at the TPDP 2021 workshop. 

Our Definition~\ref{def:reproducibility-sample}
is inspired by the literature on pseudodeterministic algorithms \cite{GG:11, GGR:12, GG:17, GGH:18, GGMW:19, GrossmanLiu:19,ECCC:Gol19}. 
In particular, \cite{GrossmanLiu:19} and \cite{ECCC:Gol19} define reproducibility in the context of pseudodeterminism. There, the input of a reproducible algorithm is a fixed string. 
In our setting, the input of a reproducible learning algorithm is a distribution, only accessible by randomly drawing samples.

Independently of our work, \cite{GKM:21} define a property equivalent to reproducibility, called ``pseudo-global stability". Their $(\alpha, \beta)$-accurate $(\eta', \nu')$-pseudo-global stability definition is equivalent to the $(\eta, \nu)$-reproducibility definition discussed in Appendix~\ref{app:additional-properties-of-reproducibility}, except that pseudo-global stability includes explicit parameters for correctness and sample complexity. 
In Appendix~\ref{app:additional-properties-of-reproducibility}, we show that these two definitions are equivalent to Definition~\ref{def:reproducibility-sample} up to polynomial factors. 
\cite{GKM:21} gives pseudo-globally stable SQ algorithms, an amplification of the stability parameter, and an algorithm to find a heavy-hitter of a distribution. The authors use pseudo-global stability to show that classes with finite Littlestone dimension can be learned user-levelly privately, and they connect pseudo-global stability to approximate differential privacy. 
Pseudo-global stability is a generalization of global stability, introduced in \cite{BLM:20}. Those authors use global stability as an intermediate step to show that classes with finite Littlestone dimension can be learned privately, and they show how global stability implies generalization.

Our work is related to other notions of stability in machine learning which,
like our definition, are properties of learning algorithms.
In the supervised learning setting, stability is a measure of how much the output of a learning algorithm changes when small changes are made to the input training set. An important body of work establishes strong connections between the stability of a learning algorithm and 
generalization \cite{DW79a,DW79b,KR99,BE02,SSSS}. 
Distributional notions of stability which remain stable under composition and postprocessing, were defined and shown to be closely connected to differential privacy and adaptive data analysis
(e.g., \cite{BNSSSU16,Dworketal:2015}). In fact, the definition of differential privacy itself is a form of stability known as max-KL stability.
Stability-based principles have also been explored in the context of
unsupervised learning where model selection is a difficult problem
since there is no ground truth.  
For example, a stable algorithm for clustering has the property that
when the algorithm is applied to different data sets from the same
distribution, it will yield similar outputs
(e.g., \cite{Luxburg2010}).

%that would , a notion of instability defined as the expected distance between two clusterings of two datasets has been used to design and analyze convergence of clustering algorithms such as $K$-means (see \cite{Luxburg2010} for a survey overview).
% Maybe could look at more specific papers and talk about them here. Reviewer #1 suggested papers by:
%Ben-David, Bubeck, Meila, von Luxburg and others from early 200s
%

%Definitions of algorithmic and distributional stability such as $\eps$-UCO stable ($\eps$-uniform change-one) and $\eps$-TV (variation distance) stable provide generalization bounds, and these distributional stability notions compose in the adaptive data analysis model.\footnote{For example, see the lecture notes in \url{https://adaptivedataanalysis.files.wordpress.com/2017/10/lect07-10-draft-v1.pdf}}

%The work of \cite{ThakurtaSmith2013} utilizes a stability notion called ``subsampling stability", getting a differentially private algorithm for computing subsampling-stable functions. Their notion defines stability as a property of (deterministic) functions when subsampling from a dataset, while our notion of reproducibility is a property of (randomized) algorithms.

In all of these settings, stability depends on
how close the outputs are when the inputs are close;
what varies is the particular measure of closeness in input and output space.
For example, closeness in the output can be with respect to function or parameter space; for
distributional stability close means that the output distributions are close
with respect to some metric over distributions.
Our definition of reproducibility can be viewed as an extreme form of
stability where the output is required to be {\it identical} almost all of the time, and not just similar. Thus reproducibility enjoys many of the nice
properties of stable algorithms (e.g., postprocessing, composition) but has the
advantage of being far easier to verify.

%Say something about how DP algs for threshold (bun, etc) use similar ideas.
%Another related notion is compression-based learning.

\paragraph{Open Questions and Future Work} One motivation for examining reproducibility in algorithms is the ``reproducibility crisis'' in experimental science.  Can we use reproducibility to create statistical methodologies that would improve reproducibility in published scientific work?  A concrete step towards this would be to
design reproducible hypothesis testing algorithms.  We can view a null hypothesis as postulating that data will come from a specific distribution $D$, and
want algorithms that accept with high probability if the data comes from $D$ (or a ``close'' distribution) and reject with good probability if the data distribution
is ``far'' from $D$.  For example, the coin problem is a degenerate case in which the data are Boolean and the distance is the difference in the expected values. For different types of data and distance metrics, what is the optimal sample complexity of hypothesis testing, and how much more is that for reproducible hypothesis testing? 

A related problem is that of learning under distributional shifts, or
individual-based fair learning (where we want the learning algorithm to
treat similar people similarly with respect to a similarity metric defining closeness).
A key step in making algorithms reproducible is a randomized procedure to round the output of a standard empirical learner to a \textit{single} hypothesis in a way that is independent of the underlying distribution.  Can similar ideas be used to design learning algorithms robust to distributional shifts, or to give more informed performance metrics?

This work establishes that there exist reproducible algorithms for a variety of learning problems.  However, we do not characterize exactly which learning algorithms
can be made reproducible, or how reproducibility affects the required sample complexity. Is it possible to identify an invariant of concept classes which characterizes the complexity of reproducible learning, analogous to VC-dimension for PAC learning \cite{vapnik71uniform}, representation dimension and one-way communication complexity for exact differential privacy \cite{COLT14:FeldmanX, ITCS13:BeimelNS13}, and Littlestone Dimension for approximate differential privacy \cite{BLM:20}? A specific problem of interest is that of learning linear functions over finite fields.  If the data has full dimension, the function can be solved for uniquely; so, designing reproducible algorithms when the data does not form a basis seems interesting. 

Also, we described the first reproducible boosting algorithm. Are there natural conditions under which a boosting algorithm can always be made reproducible? Are the sample complexity upper bounds we obtain for our applications tight or close to tight?  In particular, is there a reproducible algorithm for approximate median that has only $\log^{*}|\data|$ dependence on the domain size?  

Reproducibility provides a distinctive type of privacy. 
Except with the small probability $\rho$, a reproducible algorithm's outputs are a function entirely of the underlying distribution and the randomness of the algorithm, not the samples. Thus, a reproducible algorithm seldom leaks information about the specific input data. 
We borrowed techniques from the study of private data analysis and differential privacy, and we hope that future work will formalize connections between reproducibility and private data analysis. We also hope that some applications of differential privacy will also be achievable through reproducibility.  
%\input{prelims}

%\part{Reproducible Algorithms}
\section{Statistical Queries}
\label{sec:statistical-queries}

%\paragraph{Statistical Queries and Reproducibility.}
We show how to use randomized rounding to reproducibly simulate any SQ oracle and therefore any SQ algorithm.
The statistical query model introduced by \cite{Kearns:98} is a restriction of the PAC-learning model introduced by \cite{val84}.
We consider the statistical query oracle primarily in the context of unsupervised learning (e.g., see \cite{Feldman16}).

\begin{definition}[Statistical query oracle]
	Let $\tau \in [0,1]$ and $\phi: \data \rightarrow [0,1]$ be a query. Let $D$ be a distribution over domain $\data$. A \emph{statistical query oracle} for $D$, denoted $\mathcal{O}_{D}(\tau, \phi)$, takes as input a tolerance parameter $\tau$ and a query $\phi$, and outputs a value $v$ such that 
	$|v - \Exp_{x \sim D} [\phi(x)]| \le \tau .$
\end{definition}

\begin{definition}[Simulating a statistical query oracle ]
Let  $\delta \in [0,1]$ and $\tau, \phi, D$ be as above. Let $\mathcal{O}_D$ be a statistical query oracle for $D$. Let $\vec s$ denote an i.i.d. sample drawn from $D$. We say that a routine $\stat $ \emph{simulates} $\mathcal{O}_D$ with failure probability $\delta$ if for all $\tau, \delta, \phi$, there exists an $n_0 \in \N_{+}$ such that if $n > n_0$, $v \gets  \stat(\tau, \phi, \vec s)$ satisfies 
	$|v - \Exp_{x \sim D} [\phi(x)]| \le \tau $
	except with probability $\delta$. 
\end{definition}

To denote a routine simulating a statistical query oracle for fixed parameters $\tau, \phi$, and (optionally) $\rho$, we write these parameters as subscripts.

\begin{figure}[H]
	\begin{algorithm}[H]
		\caption{$\rstat_{\rho, \tau, \phi}(\vec s)$ 
			\\
			Parameters: 
			$\tau$ - tolerance parameter \\
			$\rho$ - reproducibility parameter \\
			$\phi$: a query $X \rightarrow [0,1]$
		}	
		\label{alg:simulate-sqs}
		\begin{algorithmic}[1]
			\STATE $\alpha = \frac{2\tau}{\rho+1-2\delta}$
			\STATE $\alphaoff \gets_r [0,\alpha]$
			\STATE Split $[0,1]$ in regions: 
				$R = \{[0,\alphaoff), [\alphaoff, \alphaoff + \alpha), \dots, [\alphaoff + i \alpha, \alphaoff + (i+1) \alpha), \dots, [\alphaoff + k\alpha, 1)\}$
			\STATE $v \gets \frac{1}{|\vec s|}\sum\limits_{x \in \vec s} \phi(x)$
			\STATE Let $r_v$ denote the region in $R$ that contains $v$
			\RETURN the midpoint of region $r_v$
		\end{algorithmic}
	\end{algorithm}
\end{figure}

Theorem~\ref{thm:reproducible-sq-oracle} upper bounds the sample complexity of $\rstat_{\tau, \rho, \phi}$. In Section~\ref{sec:coin-problem-sq-lb}, we show this upper bound is tight as a function of $\rho$. 
\begin{theorem}[$\rstat$ simulates a statistical query oracle]
	\label{thm:reproducible-sq-oracle}
	Let $\tau, \delta, \rho \in [0,1]$, $\rho > 2 \delta$, and let $\vec s$ be a sample drawn i.i.d. from distribution $D$. Then if 
	$$|\vec s| \in \tilde{O} \left(\frac{1}{\tau^2(\rho - 2\delta)^2}\right)$$  
	$\rstat_{\rho, \tau, \phi}(\vec s)$ $\rho$-reproducibly simulates an SQ oracle $\mathcal{O}_{D, \tau, \phi}$ with failure rate $\delta$. 
\end{theorem}

In Section \ref{sec:coin-problem-sq-lb}, we will prove a near
matching lower bound on the sample complexity of $\rho$-reproducibly 
estimating a statistical query with tolerance $\tau$ and success
probability $1-\delta$.

\begin{proof}
	We begin by showing that $\rstat_{\rho, \tau, \phi}$ simulates an SQ oracle $\mathcal{O}_{D, \tau, \phi}$ with failure rate $\delta$.
	
	Let $\tau' = \frac{\tau(\rho-2\delta)}{\rho + 1-2\delta}$. 
	Recall $\alpha \eqdef \frac{2\tau}{\rho + 1 - 2\delta}$, 
	so $\frac{2\tau'}{\alpha} = \rho-2\delta$. 
	A Chernoff bound gives that 
	$$ \left|\frac{1}{|\vec s |} \sum\limits_{x \in \vec s} \phi(x) - \E_{x \sim D}\phi(x) \right| \leq \tau' = \frac{\tau(\rho-2\delta)}{\rho + 1-2\delta} $$ 
	except with failure probability $\delta$, so long as $|\vec s| \geq \log(2/\delta)/(2{\tau'}^2)$. Outputting the midpoint of region $r_v$ can further offset this result by at most $\alpha/2 = \frac{\tau}{\rho+1-2\delta}$. Therefore 
	$$|v - \Exp_{x \sim D} \phi(x)| \leq \frac{\tau(\rho-2\delta)}{\rho + 1-2\delta} + \frac{\tau}{\rho+1-2\delta}  = \tau,$$
	except with probability $\delta$, so long as the sample $\vec s$ satisfies

$$
\log(2/\delta)/(2{\tau'}^2) = \frac{\log(2/\delta)(\rho + 1 - 2\delta)^2}{2\tau^2(\rho-2\delta)^2} \leq  \frac{4\log(2/\delta)}{2\tau^2(\rho-2\delta)^2}  \leq |\vec s|.
$$

	We now show that $\rstat_{ \rho, \tau, \phi}$ is $\rho$-reproducible by considering two invocations of $\rstat_{\rho, \tau,  \phi}$ with common randomness $r$ on samples $\vec s_1, \vec s_2 \sim D$ respectively.
	The probability that either empirical estimate of $\E_{x \sim D}[\phi(x)]$ fails to satisfy tolerance $\tau$ is at most $2\delta$. 
	Denote by $v_1$ and $v_2$ the values returned by the parallel runs $\rstat(\vec s_1;r)$ and $\rstat(\vec s_2; r)$ at line 4. Conditioning on success, values $v_1$ and $v_2$ differ by at most $2\tau'$. $\rstat$ outputs different values for the two runs if and only if $v_1$ and $v_2$ are in different regions of $R$, determined by the common randomness $r$. This occurs if some region's endpoint is between $v_1$ and $v_2$; since $\alphaoff$ is chosen uniformly in $[0, \alpha]$, the probability that $v_1$ and $v_2$ land in different regions is at most $2\tau'/\alpha = \rho - 2\delta$. Accounting for the $2\delta$ probability of failure to estimate $\E_{x \sim D}[\phi(x)]$ to within tolerance, $\rstat_{\rho, \tau,  \phi}(\vec s)$ is $\rho$-reproducible.  
\end{proof}

\iffalse
\sqalg
\begin{proof}
	Let $n = 1/(\tau \delta)$ and $\tau' \eqdef \tau/(nq + 1)$. 
	Apply the construction in Algorithm~\ref{alg:simulate-sqs} to an SQ oracle $\stat(D, \tau')$ with error rate $\delta' \eqdef 1/(nq)$. Using $\alpha = 2\tau'/\delta'$ yields a reproducible SQ oracle $\rstat$ with tolerance $\tau$, error $\delta'$, and reproducibility $3\delta'$. (Note: by definition, $\delta' < \delta$.) Let $A'$ be the algorithm that runs $A$ using $\rstat$. $A'$ reproduces if each call to $\rstat$ yields the same output for both executions. By a union bound over the $q$ calls to $\rstat$, $A'$ is $6/n$-reproducible. 
\end{proof}
\fi

%\section{Our Contributions}

\section{Heavy-hitters}
\label{sec:heavy-hitters}
%\paragraph{Heavy-hitters.}

Next, we present our reproducible approximate heavy-hitters algorithm, analyzing its sample complexity and reproducibility. 
We will use this algorithm as a subroutine in later
algorithms such as in the approximate-median algorithm.
Also, we will show how to use this algorithm to give a generic way to boost
reproducibility from constant $\rho$ to arbitrarily small $\rho$.
%We defer proofs of these claims to the full version of this work. 

\begin{definition}[Heavy-Hitter]
	\label{def:heavy-hitter}
	Let $D$ be a distribution over $\data$. Then we say
	$x \in \X$ is a $v$-\emph{heavy-hitter} of $D$ if
	$ \Pr_{x' \sim D} [x' = x] 
	\geq v$.
\end{definition}

\begin{definition}[(Approximate) Heavy-Hitter Problem]
	Let $L_v$ be the set of $x \in \supp(D)$ that are $v$-heavy-hitters of $D$. Given sample access to $D$, output a set $L$ satisfying $L_{v+\eps} \subseteq L \subseteq L_{v - \eps}$. 
\end{definition}

Let $D$ be a distribution over $\data$. The following algorithm reproducibly returns a set of $v'$-heavy-hitters of $D$, where $v'$ is a random value in 
$[v-\epsilon, v+\epsilon]$.
Picking $v'$ randomly allows the algorithm to, with high probability, avoid a situation where the cutoff for being a heavy-hitter (i.e. $v'$) is close to the probability mass of any $x \in \supp(D)$.
% in the support of $D$.

\begin{figure}[H]
	\begin{algorithm}[H]
		%\caption{$\rhh_{\rho, v, \epsilon}$
		\caption{$\rhhfull$
		\\Input: samples $\Xset$, $S$ from distribution $D$ over $\data$ plus internal randomness $r$
		\\Parameters: Target reproducibility $\rho$, target range $[v - \eps, v + \eps]$ 
%example oracle $\exor$ for distribution $D$ over $\data$
		\\Output: List of $v'$-heavy-hitters of $D$, where $v' \in [v - \eps, v + \eps]$
		}
		\label{alg:heavy-hitter-rep}
		\begin{algorithmic}
			\STATE $\Xset \gets Q_1 \eqdef \frac{ \ln(6/(\rho(v-\epsilon))) }{v - \epsilon}
$ examples from $D$
			\COMMENT{Step 1: Find candidate heavy-hitters}
			\STATE $S \gets 
			Q_2 \eqdef \frac{2^6\ln(Q_1/\rho)\cdot Q_1^2}{(\rho\epsilon)^2}$ 
			fresh examples from $D$
			\COMMENT{Step 2: Estimate probabilities}
			\FORALL{$x \in \Xset$}
			\STATE $\widehat{p_x} \gets \Pr_{x' \sim S}[x' = x]$
			\COMMENT{Estimate $p_x \eqdef \Pr_{x' \sim D}[x' = x]$}
			\ENDFOR
			\STATE $v' \gets_r [v - \epsilon, v + \epsilon]$
			uniformly at random
			\COMMENT{Step 3: Remove non-$v'$-heavy-hitters}
			\STATE Remove from $\Xset$ all $x$ for which $\widehat{p_x} < v'$.
			\RETURN $\Xset$
		\end{algorithmic}
	\end{algorithm}
\end{figure}

%We defer the proof to the full version, but the main idea is that

Algorithm $\rhh$ returns exactly the list of $v'$-heavy-hitters so long as the following hold:

\begin{enumerate}
	\item In Step $1$ of Algorithm~\ref{alg:heavy-hitter-rep}, all $(v-\epsilon)$-heavy-hitters of $D$ are included in $\Xset$. 
	\item In Step $2$, the probabilities $\widehat{p_x}$ for all $x \in \Xset$ are correctly estimated to within error $\rho\epsilon/(3Q_1)$. 
	\item In Step $3$, the randomly sampled $v'$ does not fall within an interval of width $\rho\epsilon/(3Q_1)$ centered on the true probability of a $(v-\epsilon)$-heavy-hitter of $D$. 
\end{enumerate} 
We show that these 3 conditions will hold with probability at least $1 - \rho/2$, and so will hold for two executions with probability at least $1 - \rho$.

\begin{lemma}
	\label{lem:rHeavyHitters-correctness}
	For all $\eps \in (0, 1/2)$, $v \in (\eps, 1- \eps)$, with probability at least $1-\rho$, 
	$\rhh$ is reproducible, returns a list of $v'$-heavy-hitters for some $v' \in [v - \eps, v + \eps]$, and has sample complexity $
	\widetilde{O}\left(
	\frac{1}
	{\rho^2\epsilon^2(v - \epsilon)^2}
	\right)$.
\end{lemma}

\begin{proof}
	We say Step $1$ of Algorithm~\ref{alg:heavy-hitter-rep} succeeds if all $(v-\epsilon)$-heavy-hitters of $D$ are included in $\Xset$ after Step $1$. Step $2$ succeeds if the probabilities for all $x \in \Xset$ are correctly estimated to within error $\rho\epsilon/(3Q_1)$. Step $3$ succeeds if the returned $\Xset$ is exactly the set of $v'$-heavy-hitters of $D$. Quantities $Q_1$ and $Q_2$ are defined in the pseudocode of Algorithm~\ref{alg:heavy-hitter-rep}.
	
	%First, we consider the probability that all $(v/2)$-heavy-hitters are included in $\Xset$ after Step 1. 
	In Step 1, an individual $(v-\epsilon)$-heavy-hitter is not included with probabilility at most 
	$(1-v + \epsilon)^{Q_1}$;
	% \le e^{-2/v}$; 
	union bounding over all $1/(v - \epsilon)$ possible $(v-\epsilon)$-heavy-hitters,
	Step 1 succeeds with probability at least 
	$1 - \frac{(1-v+\epsilon)^{Q_1}}{v-\epsilon} > 1 - \rho/6$. Here, for clarity of presentation in the statement of Lemma~\ref{lem:rHeavyHitters-correctness}, we make use of the inequality $v-\eps < \ln(1/(1-v + \eps))$.
	
	By a Chernoff bound, each $p_x$ is estimated to within error $\rho\epsilon /(3Q_1)$ with all but probability $\rho/(6Q_1)$ in Step 2. Union bounding over all $Q_1$ possible $x \in \Xset$,
	Step 2 succeeds except with probability $\rho/6$.
	
	Conditioned on the previous steps succeeding, 
	Step $3$ succeeds if the randomly chosen $v'$ is not within
	$\rho\epsilon/(3Q_1)$
	of the true probability of any $x \in \Xset$ under distribution $D$.
	A $v'$ chosen randomly from the interval $[v-\epsilon, v+\epsilon]$ lands in any given subinterval of width $\rho\epsilon /(3Q_1)$ with probability $\rho/(6Q_1)$, and so by a union bound, Step $3$ succeeds with probability at least
	$1 - \rho/6$. 
	
	Therefore, Algorithm~\ref{alg:heavy-hitter-rep} outputs exactly the set of $v'$-heavy-hitters of $D$ with probability at least $1 - \rho/2$. If we consider two executions of Algorithm~\ref{alg:heavy-hitter-rep}, both using the same shared randomness for chooosing $v'$, output the set of $v'$-heavy-hitters of $D$ with probability at least $1 - \rho$, and so $\rhh$ is $\rho$-reproducible. 
	%
	
%	We let
%	$Q_2 \eqdef \frac{2^6\ln(Q_1/\rho)\cdot Q_1^2}{(\rho\epsilon)^2}$,
%	and
%	$Q_1 \eqdef \frac{ \ln(6/(\rho(v-\epsilon))) }{v - \epsilon}
%	 $, so
	 The sample complexity is
	 $
	 Q_1 + Q_2 \in \tilde{\Omega}\left(
	 (\rho\epsilon(v - \epsilon))^{-2}
	 \right)
	 $.
\end{proof}

\begin{corollary}
	\label{cor:rHeavyHitters-sample-complexity-constants}	
	If $v$ and $\eps$ are constants, then
	$\rhh_{\rho,  v,\eps}$ has sample complexity $ 
	\widetilde{O}\left( 
	1/\rho^2
	\right)$.
\end{corollary}

\paragraph{Learning Heavy-hitters using Statistical Queries.}

Next, we show that any statistical query algorithm for the $v$-heavy-hitters problem requires 
$\Omega(\log |\data|/\log(1/\tau))$
calls to the SQ oracle. Since Algorithm~\ref{alg:heavy-hitter-rep} has a sample complexity independent of the domain size, this implies a separation between reproducible problems and problems solvable using only SQ queries. 
%We defer the proof of Claim~\ref{claim:sq-heavy-hitters} to this paper's full version. 

	Consider the ensemble $\{D_x\}_{x\in \data}$ on $\data$, where distribution $D_x$ is supported entirely on a single $x \in \data$.

\begin{claim}[Learning Heavy-hitters using Statistical Queries]\label{claim:sq-heavy-hitters}
	Any statistical query algorithm for the $v$-heavy-hitters problem on ensemble $\{D_x\}_{x\in \data}$ requires 
	$\Omega(\log |\data|/\log(1/\tau))$ 
	calls to the SQ oracle.
\end{claim}

%\iffalse
\begin{proof}
 An SQ algorithm for the $v$-heavy-hitters problem must, for each distribution $D_x$, output set $\{x\}$ with high probability. An SQ oracle is allowed tolerance $\tau$ in its response to statistical query $\phi$. So, for any $\phi$, there must be some distribution $D_x$ for which the following holds: at least a $\tau$-fraction of the distributions $D_{x'}$ in the ensemble satisfy $|\phi(x') - \phi(x)| \le \tau$. Thus, in the worst case, any correct SQ algorithm can rule out at most a $(1-\tau)$-fraction of the distributions in the ensemble with one query. 
	If $\data$ is finite, then an SQ algorithm needs at least
	$\log_{1/\tau} (|\data|)$ queries. 
\end{proof}
%\fi

\section{Approximate Median}
\label{sec:approximate-median}
In this section, we design a reproducible algorithm for 
finding an approximate median in an arbitrary 
distribution over a finite domain. In addition to being a significant
problem in its own right, and one studied extensively in the privacy
literature, this is a key sub-routine for making many algorithms
reproducible.  In particular, for any problem where the correct answers
form an interval, and we have a (not-necessarily) reproducible algorithm
that is correct strictly more than half the time, we can run the approximate
median finding algorithm on the distribution of outputs of the original
to construct a reliably correct and reproducible version.  (In fact, we
use this technique recursively within our reproducible median-finding algorithm
itself. 
Our algorithm $\rmedofmedians$ composes a mildly accurate reproducible median algorithm with a generic very accurate non-reproducible median algorithm.) We use
a recursive technique inspired by but simpler than  
previous algorithms in the privacy literature \cite{BNSV:15, KLMNS20}, and like for these
algorithms,  the sample complexity
of our algorithm has a non-constant but very slowly growing dependence on
the domain size.   
%We also can use this algorithm to reproducibly PAC-learn (one-dimensional) 
%threshold functions, using a known reduction from \cite{BNSV:15}.

\begin{definition}[$\tau$-approximate median]
	\label{def:approximate-median}
	Let $D$ be a distribution over a well-ordered domain $\data$.
	$x \in \data$ is a $\tau$-approximate median of $D$ if 
	$\Pr_{x' \sim D} [x' \leq x] \ge 1/2 - \tau$
	and 
	$\Pr_{x' \sim D} [x' \geq x] \ge 1/2 - \tau$.
\end{definition}

\subsection{Reproducible Approximate Median Algorithm}

In this section, we present a pseudocode description of our $\tau$-approximate median algorithm $\rmedian$ (Algorithm~\ref{alg:rep-median}), and prove the following theorem.

\begin{restatable}[Reproducible Median]{theorem}{repmedian}
	\label{thm:median}
		Let $\tau, \rho \in [0,1]$ and let $\delta = \rho/2$. Let $D$ be a distribution over $\data$, where $|\data| = 2^d$. Then $\rmedian_{\rho, d, \tau, \delta}$ (Algorithm~\ref{alg:rep-median}) is $\rho$-reproducible, outputs a $\tau$-approximate median of $D$ with all but probability $\delta$, and has sample complexity 
	$$ n\in \tilde{O}\left( \left( \frac{1}{\tau^2(\rho - \delta)^2}\right) \cdot \left(\frac{3}{\tau^2}\right)^{\log^{*}|\data|} \right)$$
\end{restatable}

As an introduction to the key ideas of Algorithm~\ref{alg:rep-median}, we consider a weighted binary tree $T$ based on distribution $D$. 
Each internal node has two edges (a $0$-edge and a $1$-edge). Root-to-leaf paths represent binary representations of numbers. The weight of each internal node $v$ is the probability that its associated binary prefix (induced by the root-to-$v$ path) appears in an element drawn from $D$.
If within this tree we can find a node $v$ with weight in $[1/4, 3/4]$, then we can use the associated prefix to return an approximate median of $D$ with approximation parameter potentially much larger than $\tau$. 

To achieve a specified approximation parameter $\tau$, rather than using $D$ itself to construct the binary tree $T$, we will use a distribution $D^m$ over medians of $D$. Specifically, we use a non-reproducible median algorithm to sample from $\tau$-approximate medians of $D$. Identifying an approximate median of distribution $D^m$ for even a very large approximation parameter then ensures we return a $\tau$-approximate median of $D$.

The question remains of how to efficiently search $T$ to find a node $v$ of weight in $[1/4,3/4]$ (under $D^m$). We perform this search recursively by using $\rmedian$ to find a prefix length $\ell$ such that the probability of sampling two elements from $D^m$ agreeing on a prefix of length $\ell$ is large. We can then restrict our search for $v$ to nodes near level $\ell$ in $T$ (starting from the root). We apply the reproducible heavy-hitters algorithm $\rhh$ to find high weight nodes near level $\ell$ of $T$, and then exhaustively search the list of heavy-hitters to find an appropriate $v$.  

We use the following non-reproducible approximate median algorithm, that returns the median of its sample $\vec s$, as a subroutine of Algorithm~\ref{alg:rep-median}.

\begin{lemma}[Simple Median Algorithm]
	\label{lem:simple-median}
	Let sample $\vec s$ be drawn from distribution $D$.
	% over $\twotothed$. 
	Algorithm $\median(\vec s)$ returns a $\tau$-approximate median on $D$ using $|\vec s| = 3(1/2-\tau)\ln(2/\delta)/\tau^2$ samples with success probability at least $1- \delta$. 
\end{lemma}

\begin{proof}
	Algorithm $\median(\vec s)$ fails when more than half of the elements in sample $\vec s$ are either i) smaller than the $(1/2-\tau)$-percentile element of $D$ or ii) larger than the $(1/2+\tau)$-percentile element of $D$. Let event $E_i$ denote the first case and event $E_{ii}$ denote the second case. 
	Since the elements in $\vec s$ are drawn i.i.d., the first event can be bounded by a Chernoff bound. Let $X$ be a random variable denoting the number of elements in $\vec s$ that are smaller than the $(1/2-\tau)$-percentile element of $D$. 
	\begin{equation*}
		\begin{split}
			\Pr[E_i] 
			&= \Pr[X \ge (1 + \tau/(1/2 - \tau)) \E[X] ] 
			\\&\le \exp( -(\tau/(1/2 - \tau))^2 \E[X]/3 )
			\\&\le \exp \left( -\frac{\tau^2}{1/2-\tau}\frac{|\vec s|}{3} \right) = \exp (-\ln(2/\delta))
			\\&= \delta/2
		\end{split}
	\end{equation*}
	The same argument can be used to bound the second event $E_{ii}$. By a union bound, the algorithm succeeds with probability at least $1-\delta$.
\end{proof}

Before proceeding with the description of Algorithm~\ref{alg:rep-median}, we fix some useful notation for its analysis.
\begin{itemize}
	\item $n_m $ - sample complexity of $\median_{\tau, \delta_0}$
	\item $n_h$ - sample complexity of $\rhh_{\rho_0, v, \epsilon }$
	\item $n_{sq}$ - sample complexity of $\rstat_{ \tau, \rho_0, \phi}$
	\item $n_{d}$ - sample complexity of $\rmedian_{\rho, d, \tau, \delta}$
	\item $D^m$ - Algorithm~\ref{alg:rep-median} takes as input a sample from distribution $D$ over $\data$, where $|\data| = 2^d$. We use $D^m$ to denote the distribution induced by sampling $n_m$ examples from $D$, computing $\median_{\tau, \delta_0}$ on these examples, and returning the ouput
	\item $D_{\lceil \log d \rceil}$ - We use $D_{\lceil \log d \rceil}$ to denote the distribution induced by sampling 2 examples from $D^m$ and returning the longest prefix $\ell$ on which the two medians agree. Note that this new distribution is over a new domain $\data'$ with $|\data'| =  2^{\lceil \log d \rceil } \in \Theta(d) $.
	\item $\rho_0 \in O(\rho/\log^{*}|X|)$
	\item $\delta_0 \in O\left((\frac{\delta}{n_h + n_{sq}})^{2\log^{*}|\data|} \cdot \left(\frac{\tau^2}{3}\right)^{2({\log^{*}}|\data|)^2}\right) $
\end{itemize}

\begin{algorithm}[H]
	\begin{algorithmic}[1]
		\caption{$\rmedian(\vec s)$ \\
			Input: $\vec s$ - a sample of $n$ elements drawn i.i.d. from $D$ \\
			Parameters: 
			\\$\rho$ - target reproducibility parameter 
			\\ $d$ - specifies domain size $| \data | = 2^d$ 
			\\ $\tau$ - target accuracy of median 
			\\ $\delta$ - target failure probability
			\\ Output: a $\tau$-approximate median of $D$
		}\label{alg:rep-median}
	\IF{$d = 1$}
	\STATE Let $\phi_0 (x) = \begin{cases} 1 & \text{ if } x = 0 \\ 0 & o/w\end{cases}$ 
	\STATE $p_0 \gets \rstat_{\rho_0, \tau/2, \phi_0}(\vec s)$ \label{line:basecaseq}
	\COMMENT{Base case}
	\IF{$p_0 \geq 1/2 - \tau/2$}
	\RETURN $0$
	\ELSE
	\RETURN $1$
	\ENDIF
	\ENDIF 
	\STATE Break $\vec s$ into $|\vec s|/n_m$ subsamples
	\STATE Run $\median_{\tau, \delta_0}$ on each subsample to generate a new sample $\vec m$ of $\tau$-approximate medians of $D_{d}$ \label{line:medians}
	\STATE Pair up elements $\vec m_{2i}$ and $\vec m_{2i-1}$, for $i \in \{1, \cdots, |\vec m|/2\}$
	\STATE For each pair $(\vec m_{2i}, \vec m_{2i-1})$, let $l_i$ denote the longest prefix on which they agree \label{line:genl}
	\STATE Let $\vec s_{rm}$ denote the multiset of $l_i$'s
	\STATE $\ell \gets \rmedian_{\rho, \lceil \log d \rceil, \tau, \delta}(\vec s_{rm})$ \label{line:recursive}
	\STATE $\vec s_{h0}, \vec s_{h1} \gets n_h$ new examples from $\vec m$ each
	\STATE $\vec s_{\ell} \gets \{x_{|\ell} : x\in \vec s_{h0}\}$
	\COMMENT{$\vec s_{\ell}$ is the set $\vec s_{h0}$ projected onto length $\ell$ prefixes}
	\STATE $V \gets \rhh_{\rho_0, v, \epsilon }(\vec s_{\ell})$, for $v = 5/16 + \tau$, $\epsilon = 1/16$ \label{line:hh1}
	\IF{$\ell < d$}
	\STATE $\vec s_{\ell + 1} \gets \{x_{|\ell + 1} : x \in \vec s_{h1}\}$
	\STATE $V \gets V \cup \rhh_{\rho_0, v, \epsilon }(\vec s_{\ell + 1})$ \label{line:hh2}
	\COMMENT{Find vertices at level $\ell$ and $\ell + 1$ with weight $\geq 1/4$}
	\ELSE
	\RETURN the first element of $V$ \label{line:lisd}
	\ENDIF 
	\FOR{$v \in V$}\label{line:sqloop}
	\STATE Let $\phi_v(x) = \begin{cases} 1 & \text{ if } x_{||v|} = v \\ 0 & \text{o/w} \end{cases}$
	\STATE $\vec s_q \gets n_{sq}$ new examples from $\vec m$
	\STATE $p_v \gets \rstat_{\rho_0, \tau, \phi_v}(\vec s_q)$, 
	\COMMENT{Query $D^{m}$ for probability $x \leq v||1 \cdots 1$}
	\IF{$1/4 \leq p_v \leq 3/4 $}
	\STATE $s \gets v$
	\COMMENT{Find length $\ell$ prefix of weight in $[1/4 - \tau, 3/4 + \tau]$}
	\ENDIF 
	\ENDFOR 
	\STATE $s_0 = s || 0 \cdots 0$ 
	\COMMENT{$s_0$ is the prefix $s$ padded with $0$'s to length $d$}
	\STATE $s_1 = s || 1 \cdots 1$ 
	\COMMENT{$s_1$ is the prefix $s$ padded with $1$'s to length $d$}
	\STATE 
	\STATE Let $\phi_{s_0}(x) = \begin{cases} 1 & \text{ if } x \leq s_0 \\ 0 & \text{o/w}
		  \end{cases}$  
	\STATE $\vec s_{s_0} \gets n_m$ new examples from $\vec m$
	\STATE $p_{s_0} \gets \rstat_{  \rho_0, \tau, \phi_{s_0}}(\vec s_{s_0})$  \label{line:choosemedianq}
	\IF{$p_{s_0} \geq 1/8 - 2\tau $}
	\RETURN $s_0$ \label{line:returnmedian}
	\ELSE
	\RETURN $s_1$
	\ENDIF
	\end{algorithmic}
\end{algorithm}

\begin{lemma}[Termination]\label{lem:termination}
	Algorithm~\ref{alg:rep-median} terminates after $T = \log^{*}|\data|$ recursive calls. 
\end{lemma}

\begin{proof}
	Algorithm~\ref{alg:rep-median} reaches its base case when invoked with parameter $d = 1$. At each successive recursive call (Line~\ref{line:recursive}), the domain size $2^d$ is reduced to $2^{\lceil \log d \rceil} < 2d$, and so $d=1$ after no more than $T = \log^{*}|\data|$ recursive calls.
\end{proof}

\begin{lemma}[Sample Complexity]\label{lem:sample-complexity}
	Let $\tau, \delta, \rho \in [0,1]$. Let $D$ be a distribution over $\data$, with $|\data| = 2^d$. Then
	$\rmedian_{\rho, d, \tau, \delta}$ has sample complexity 
		$$n \in O\left( \left( \frac{1}{\tau^2(\rho - \delta)^2}\right) \cdot \left(\frac{3\log(2/\delta_0)}{\tau^2}\right)^{\log^{*}|\data|} \right)$$
\end{lemma}

\begin{proof}
	We begin by arguing that, for $d > 1$, $\rmedian_{\rho, d, \tau, \delta}$ has sample complexity $n_m(2n_{\lceil \log d \rceil} + n_h +  4n_{sq})$.
	First, observe that Line~\ref{line:medians} of Algorithm~\ref{alg:rep-median} is the only line that uses the sample $\vec s$ directly, and it uses $\vec s$ to generate a sample $\vec m$ of size $|\vec s|/n_m$ from $D^m$. The remaining subroutines use subsamples from $\vec m$. Therefore, if the sample complexity of the remaining subroutines is bounded by some value $N$, then $\rmedian_{\rho, d, \tau, \delta}$ will have sample complexity $Nn_m$. We now consider the sequence of subroutines and their respective complexities. 
	
	\begin{enumerate}
		\item Line~\ref{line:recursive}: $\rmedian_{\rho, \lceil \log d \rceil, \tau, \delta}$ requires $n_{\lceil \log d \rceil}$ examples from $D_{\lceil \log d \rceil}$. Line~\ref{line:genl} generates an example from $D_{\lceil \log d \rceil}$ from 2 examples from $D^m$, and so the call to $\rmedian_{\rho, \lceil \log d \rceil, \tau, \delta}$ at Line~\ref{line:recursive} contributes $2n_{\lceil \log d \rceil}$ to the sample complexity.
		\item Line~\ref{line:hh1} and Line~\ref{line:hh2}: $\rhh_{\rho_0, v, \epsilon }$ requires $n_h$ examples from $D^m$
		\item Line~\ref{line:sqloop}: the at most 3 calls to $\rstat_{\rho_0, \tau, \phi_v}$ require $3n_{sq}$ examples from $D^m$
		\item Line~\ref{line:choosemedianq}: $\rstat_{  \rho_0, \tau,\phi_{s_0}}$ requires $n_{sq}$ examples from $D^m$ 
	\end{enumerate}
	Therefore $\rmedian_{\rho, d, \tau, \delta}$ uses $n = n_m(2n_{\lceil \log d \rceil} + 2n_h + 4n_{sq})$ examples from $D$. 
	
	In the base case, the entire contribution to the sample complexity comes from the call to $\median_{\tau, \delta_0}$, which requires $n_m$ examples from $D_1$. Unrolling the recursion, we have
	
	\begin{equation*}
		\begin{split}
			n 
			&\in O\left( (2n_m)^{\log^{*}|\data|}(n_h + n_{sq}) \right) 
			\\&\in \tilde{O}\left( \left( \frac{1}{\tau^2(\rho - \delta)^2}\right) \cdot \left(\frac{3\log(2/\delta_0)}{\tau^2}\right)^{\log^{*}|\data|} \right).
		\end{split}
	\end{equation*}
	
\end{proof}

\begin{lemma}[Accuracy]\label{lem:accuracy}
	Let $\rho, \tau, \delta \in [0,1]$ and let $n$ denote the sample complexity proved in Lemma~\ref{lem:sample-complexity}. Let $\vec s$ be a sample of elements drawn i.i.d. from $D$ such that $|\vec s| \in \Omega(n)$. Then $\rmedian(\vec s)$ returns a $\tau$-approximate median of $D$ except with probability $\delta$. 
\end{lemma}

\begin{proof}
	First, we prove that $\rmedian(\vec s)$ returns a $\tau$-approximate median of $D$, conditioned on the success of all recursive calls and subroutines.  
	We proceed inductively. In the base case we have that $|\data| = 2$, and therefore at least one of the two elements in $\data$ must be a $\tau$-approximate median. The statistical query performed in line 6 of Algorithm~\ref{alg:rep-median} uses sample $\vec s$ to estimate the fraction of $D_1$ supported on $0$, to within tolerance $\tau/2$, so long as $|\vec s| \geq n_m$. This holds from Lemma~\ref{lem:sample-complexity}, and so a $\tau$-approximate median for $D_1$ is returned in the base case.
	
	It remains to show that if a $\tau$-approximate median for $D_{\lceil \log d \rceil}$ is returned at Line~\ref{line:recursive} of Algorithm~\ref{alg:rep-median}, that a $\tau$-approximate median for $D$ is returned.  
	We first note that, except with probability $\delta_0 \cdot |\vec s|/n_m$, all elements of $\vec m$ are $\tau$-approximate medians of $D$. To generate the sample supplied to $\rmedian$ at Line 13, we pair up the elements of $\vec m$ to obtain the $|\vec s|/(2n_m)$ $l_i$, which denote the longest prefix on which a pair of elements from $\vec m$ agree. Then $\vec s_{rm}$ constitutes a sample of size $n_{\lceil \log d \rceil}$ drawn i.i.d. from $D_{\lceil \log d \rceil}$ and by inductive assumption the call to $\rmedian$ at Line~\ref{line:recursive} returns a $\tau$-approximate median of $D_{\lceil \log d \rceil}$. Therefore we have that 
	$\pr_{x_1, x_2 \sim D^{m}}[x_{1 | \ell} = x_{2 | \ell}] \geq 1/2 - \tau$ and $\pr_{x_1, x_2 \sim D^{m}}[x_{1 | \ell + 1} = x_{2 | \ell + 1}] <  1/2  + \tau$. 
	It follows that there must exist a prefix $s$ of length $\ell$ such that $\Pr_{x\sim D^{m}}[x_{|\ell} = s] \geq 1/4$.
	
	If $\ell = d$, then $x_{|\ell} = x$, and so any prefix $s$ such that $\Pr_{x\sim D^{m}}[x_{|\ell} = s] \geq 1/4$ is a $3/8$-median of $D^{m}$ and therefore a $\tau$-median of $D$. In this case $s$ is returned at Line~\ref{line:lisd}. 
	
	For the remainder of the proof, we assume $\ell < d$. We argue that there must exist a prefix $s$ of length $\ell$ or $\ell + 1$ for which $1/4 \leq \Pr_{x \sim D^{m}}[x_{| |s|} = s] \leq 3/4$. We already have that there exists a prefix $s$ of length $\ell$ such that $\Pr_{x\sim D^{m}}[x_{\ell} = s] \geq 1/4$. Suppose that $\Pr_{x \sim D^{m}}[x_{| \ell} = s] > 3/4$. Now suppose that one of $s||0$ or $s||1$ had probability greater than $3/4$ under $D^{m}$. Then it must be the case that $\Pr_{x_1, x_2 \sim D^{m}}[x_{1 |\ell + 1} = x_{2 | \ell + 1}] > 9/16$, and so $\Pr_{\ell' \sim D_{\lceil \log d \rceil}}[\ell' \leq \ell] < 1 - 9/16 = 7/16$, contradicting that $\ell$ is a $\tau$-approximate median of $D_{\lceil \log d \rceil}$. So both $s||0$ and $s||1$ must have probability less than 3/4 under $D^{m}$. Because $s$ has probability at least 3/4, it follows that at least one of $s||0$ and $s||1$ must have probability at least 1/4 under $D^{m}$, and so we have that there exists a prefix $s'$ of length $\ell+1$ such that $1/4 \leq \Pr_{x \sim D^{m}}[x_{|\ell+1} = s'] \leq 3/4$.
	
	Now that we have the existence of such a prefix, we will argue that when the loop of Line~\ref{line:sqloop} terminates, $s$ is a prefix satisfying
	$$1/4 - \tau \leq \pr_{x_1 \sim D^{m}}[x_{1 | \ell} = s] \leq 3/4+ \tau.$$ 
	Observe that the calls to $\rhh$ at Line~\ref{line:hh1} and Line~\ref{line:hh2} identify a common prefix $s$ such that $\pr_{x_1\sim D^{m}}[x_{1 | \ell} = s] \geq 1/4$. This follows from taking $v = 5/16$, $\epsilon = 1/16$, and the fact that the sample $\vec s_{\ell}$ and $\vec s_{\ell+1}$ constitute i.i.d. samples of size $n_h$ drawn from $D^m_{|\ell}$ and $D^m_{|\ell +1}$ respectively (where we use $D^{m}_{|\ell}$ to indicate the distribution induced by sampling from $D^{m}$ and returning only the first $\ell$ bits). Then we have from the proof of Lemma~\ref{lem:rHeavyHitters-correctness} that all $v-\epsilon = 1/4$-heavy hitters from $D^m_{|\ell}$ and $D^m_{|\ell +1}$ are contained in the set $V$.
	The loop beginning at Line~\ref{line:sqloop} will use reproducible statistical queries to estimate the probability of each $v \in V$ under $D^{m}_{| |v|}$. If the estimated probability $p_v \in [1/4, 3/4]$, then $v$ is stored in $s$, and so the last such string visited by the loop is the value of $s$ upon termination. 
	
	Now we show that if $s_0 = s || 0 \cdots 0$ is returned at Line~\ref{line:returnmedian}, then it is a $\tau$-approximate median of $D$, otherwise $s_1 = s || 1 \cdots 1$ is a $\tau$-approximate median. Conceptually, we partition the domain $\data$ into three sets: 
	\begin{enumerate}
		\item $C_{s_0}  = \{ x \in \data : x < s_0\}$
		\item $C_{s}  = \{ x \in \data : s_0 \leq x \leq s_1 \}$
		\item $C_{s_1}  = \{ x \in \data : x > s_1\}$
	\end{enumerate}   
	Because $s$ satisfies $1/4 - \tau \leq \pr_{x \sim D^{m}}[x_{| |s|} = s] \leq 3/4 + \tau$, it must be the case that $D^m$ assigns probability mass at least $1/4 - \tau$ to the union $C_{s_0} \cup C_{s_1}$. Then it holds that at least one of $C_{s_0}$ and $C_{s_1}$ is assigned probability mass at least $1/8 - \tau/2$. The statistical query made at Line~\ref{line:choosemedianq} estimates the probability mass assigned to $C_{s_0}$ by $D^m$ to within tolerance $\tau$, so if $s_0$ is returned, it holds that $\Pr_{x \sim D^m}[x < s_0] \geq 1/8 - 3\tau$. Because we know $\Pr_{x \sim D^m}[x \in C_s] \geq 1/4 - \tau$, we then also have that $\Pr_{x \sim D^m}[x \geq s_0] \geq 1/4 - \tau$. Because $D^m$ is a distribution over $\tau$-approximate medians of $D$, we have that $s_0$ is a $\tau$-approximate median of $D$ as desired. If $s_0$ is not returned, then it must be the case that $\Pr_{x \sim D^m}[x > s_1] \geq 1/8 - 3\tau$, and a similar argument shows that $s_1$ must be a $\tau$-approximate median of $D$.
	
	Finally, we argue that all recursive calls and subroutines are successful, except with probability $\delta$. Failures can occur exclusively at the following calls.
	\begin{itemize}
		\item Line~\ref{line:medians}: the $\log^{*}|\data|\cdot |\vec s|/(n_m)$ calls to $\median_{\tau, \delta_0}$ 
		\item Line~\ref{line:recursive}: the $\log^{*}|\data|$ recursive calls to $\rmedian_{\rho, \lceil \log d \rceil, \tau, \delta}$
		\item Line~\ref{line:hh1} and Line~\ref{line:hh2}: the $2\log^{*}|\data|$ calls to $\rhh_{\rho_0, v, \epsilon }$
		\item Line~\ref{line:sqloop}: the (at most) $4\log^{*}|\data|$ calls to $\rstat_{ \rho_0, \tau, \phi_{v}}$
		\item Line~\ref{line:choosemedianq}: the $\log^{*}|\data|$ calls to $\rstat_{ \rho_0, \tau, \phi_{s_0}}$
	\end{itemize}
	Calls to $\median_{\tau, \delta_0}$ dominate the total failure probability, and so taking 
	$\delta_0 \in O( \frac{\delta }{|\vec s| \log^{*}|\data|})$ suffices to achieve failure probability $\delta$.  
\end{proof}

\begin{lemma}[Reproducibility]\label{lem:reproducibility}
	Let $\rho, \tau, \delta \in [0,1]$ and let $n$ denote the sample complexity proved in Lemma~\ref{lem:sample-complexity}. Let $\vec s$ be a sample of $O(n)$ elements drawn i.i.d. from $D$. Then $\rmedian_{\rho, d, \tau, \delta}$ is $\rho$-reproducible. 	
\end{lemma}

\begin{proof}
	We prove the lemma by inductive argument. First, we observe that reproducibility of the value returned in the base case depends only on the value $p_0 \gets \rstat_{\rho_0, \tau/2 ,\phi_0}(\vec s)$ in Line~\ref{line:basecaseq}. Therefore, reproducibility in the base case follows from the $\rho_0$-reproducibility of $\rstat_{\rho_0, \tau/2, \phi_0}$.
	
	 We now argue that if the $i+1$th recursive call is $\rho$-reproducible, that the $i$th recursive call is $(\rho + 5\rho_0)$-reproducible. 
	
	Two parallel executions of the $i$th level of recursion, given samples $\vec s_1$ and $\vec s_2$ drawn i.i.d. from the same distribution $D$, will produce the same output so long as the following values are the same:
	\begin{enumerate}
		\item $\ell \gets \rmedian_{\rho, d, \tau, \delta}(\vec s_{rm})$ at Line~\ref{line:recursive}
		\item $V \gets \rhh_{\rho_0, v, \epsilon }(\vec s_{\ell})$ at Line~\ref{line:hh1}
		\item $V \gets V \cup \rhh_{\rho_0, v, \epsilon }(\vec s_{\ell + 1})$ at Line~\ref{line:hh2}
		\item $s \gets \rstat_{ \rho_0,\tau, \phi_{s_0}}(\vec s_{meds} )$ when the loop at Line~\ref{line:sqloop} terminates
		\item $p_{s_0} \gets \rstat_{ \rho_0, \tau, \phi_{s_0}}(\vec s_{s_0})$ at Line~\ref{line:choosemedianq}
	\end{enumerate}
	produce the same value. We have that 1 holds by inductive assumption. 
	
	Conditioning on 1, the calls to $\rhh_{\rho_0, v, \epsilon }$ are made on samples drawn i.i.d. from the same distribution, and so the $\rho_0$-reproducibility of $\rhh_{\rho_0, v, \epsilon }$ guarantees that $V$ contains the same list of heavy-hitters in both runs except with probability $2\rho$. 
	
	Conditioning on both 1 and 2, it follows that the loop at Line~\ref{line:sqloop} iterates over the same strings $V$, and so both runs make the same sequence of statistical queries $\rstat_{\tau, \rho_0, \phi_v}$. From conditioning on 2, and the values of $v$ and $\epsilon$, we have that $|V| \leq 3$, and so the $\rho_0$-reproducibility of $\rstat_{\tau, \rho_0, \phi_{s_0}}$ gives us that sequence of values $p_v \gets \rstat_{\rho_0, \tau,\phi_{v}}(\vec s_q)$ is the same in both runs, except with probability $3_{\rho_0}$. 
	
	Finally, conditioning on 1, 2, and 3, the values of $s_0$ and $s_1$ are the same across both runs, and so the same statistical query $\rstat_{\rho_0, \tau, \phi_{s_0}}$ is made in both runs. Whether $s_0$ or $s_1$ is returned depends only on the value $r_{s_0} \gets \rstat_{ \tau, \rho_0, \phi_{s_0}}(\vec s_{s_0})$, and so the $\rho_0$-reproducibility of $\rstat_{ \rho_0,\tau, \phi_{s_0}}$ gives us that the same string is returned by both executions. A union bound over all failures of reproducibility then gives us that the $i$th recursive call will be $(\rho + 6\rho_0)$-reproducible. 
	
	From Lemma~\ref{lem:termination}, we have that no more than $T = \log^{*}|\data|$ recursive calls are made by the algorithm. Therefore $\rmedian_{\rho, d, \tau, \delta}$ is reproducible with parameter $\rho_0 + 5T\rho_0 \leq 6\rho_0\log^*|\data| = \rho$.
\end{proof}

Theorem~\ref{thm:median} then follows as a corollary of Lemma~\ref{lem:sample-complexity}, Lemma~\ref{lem:accuracy}, and Lemma~\ref{lem:reproducibility}.

\repmedian

\section{Learning Halfspaces}
\label{sec:halfspaces}

In Section~\ref{sec:statistical-queries}, we saw how combining a concentration bound with a randomized rounding technique yielded a reproducible algorithm. Specifically, given a statistical query algorithm with an accuracy guarantee (with high probability) on the 1-dimensional space $[0,1]$, we can construct a reproducible statistical query algorithm using randomized rounding. By sacrificing a small amount of accuracy, our reproducible statistical query algorithm can decide on a canonical return value in $[0,1]$. 

In this section, we extend this argument from $\R$ to $\R^d$, by way of an interesting application of a randomized rounding technique from the study of foams \cite{KDRW12}. 
Algorithm~$1$ in \cite{KDRW12} probabilistically constructs a tiling of $\R^d$ such that every point is rounded to a nearby integer lattice point. This tiling has an additional property that the probability that two points are not rounded to the same point by a constructed tiling is at most linear in their $l_2$ distance. 
In the usual PAC-learning setting, there is a simple weak learning algorithm for halfspaces that takes examples $(\vec{x}_i,y_i) \in \data \times \pmone$, normalizes them, and returns the halfspace defined by vector $\sum_i \vec{x}_i \cdot y_i$ \cite{Ser02PAC}. 
We show a concentration bound on the sum of normalized vectors from a distribution, and then argue that all vectors within the concentration bound are reasonable hypotheses with non-negligible advantage.
The combination of this concentration bound and the foam-based rounding scheme yields a reproducible halfspace weak learner $\rhalfspacewkl$.

However, constructing this foam-based rounding scheme takes expected time that is exponential in the dimension $d$. We give an alternative rounding scheme that randomly translates the integer lattice and rounds points to their nearest translated integer lattice point. This construction yields another reproducible halfspace weak learner $\rhalfspacewklboxes$ with roughly an additional factor of $d$ in the sample complexity, but with polynomial runtime.  
In Section~\ref{sec:boosting}, we show how to combine these reproducible weak learners with a reproducible boosting algorithm, yielding an polynomial-time reproducible strong learner for halfspaces.

\subsection{Reproducible Halfspace Weak Learner: An Overview}

Let $D$ be a distribution over $\R^d$, and let $\exor$ be an example oracle for $D$ and $f$, where $f: \R^d \rightarrow \pmone$ is a halfspace that goes through the origin. 
Let $\| \vec{x}\|$ denote the $l_2$ norm of vector $\vec{x}$. 
We assume that $D$ satisfies a (worst-case) margin assumption with respect to $f$. 

\begin{restatable}{definition}{margindef}[Margin]
	Let $D$ be a distribution over $\R^d$. We say $D$ has margin $\tau_f$ with respect to halfspace 
	$f(\vec{x}) \eqdef \sgn(\vec{w} \cdot \vec{x})$ if 
	$
	%\left|
	\frac{\vec{x}
		\cdot f(\vec{x}) 
	}{\|\vec{x}\|} 
	\cdot \frac{ \vec{w} }{ \|\vec{w}\|} 
	%\right| 
	\ge \tau_f$ for all $x \in \supp(D)$. 
	Additionally, we say $D$ has (worst-case) margin $\tau$ if
	$\tau = \sup_f \tau_f$.
\end{restatable}

Our reproducible halfspace weak learner $\rhalfspacewkl$ uses its input to compute an empirical estimation $\vec{z}$ of the expected vector $\Exp_{\vec{x} \sim D}[\vec{x} \cdot f(x)]$. Then, $\rhalfspacewkl$ uses its randomness to construct a rounding scheme $R$ via Algorithm~$\constructfoam$. $R$ is used to round our (rescaled) empirical estimation $\vec{z}$, and the resulting vector defines the returned halfspace. The algorithm relies on the margin assumption to ensure that the weak learner's returned hypothesis is positively correlated with the true halfspace $f$.\footnote{The parameter $a$ is a constant, but we leave it in variable form for convenience in the analysis; we choose $a = .05$ in this proof for clarity of presentation, but one could optimize the choice of $a$ to yield a slightly better sample complexity.}

\begin{figure}[H]
	\begin{algorithm}[H]
		\caption{$\rhalfspacewklfull$ 
			~ // a $\rho$-reproducible halfspace weak learner
			\\Parameters: $\rho$ - desired reproducibility 
			\\ $d$ - dimension of halfspace
			\\ $\tau$ -  assumed margin
			\\ $a$ - a constant, $a = .05$
			\\Input: A sample $\vec s$ of 
			$m = 
			\left( \frac{896 \sqrt{d}}{\tau^2 \rho} \right)^{1/(1/2-a)}$ 
			examples $(\vec{x}_i, y_i)$ drawn i.i.d. from distribution $D$
			%. Access to example oracle $\exor(D,f)$.
			\\Output: A hypothesis with advantage $\tau/4$ on $D$ against $f$
		}	
		\label{alg:reproducible-halfspace}
		\begin{algorithmic}
%			\STATE $a \gets .05$
%			\COMMENT{A constant in $(0,1/2)$}
%			\STATE $m \gets 
%			\left( \frac{896 \sqrt{d}}{\tau^2 \rho} \right)^{1/(1/2-a)}
%			$
%			\COMMENT{Sample complexity}
			\STATE $k \gets \frac{1}{m} \frac{8\sqrt{d}}{\tau^2}
			= 8 \cdot 
			\left( \frac{\rho}{896} \right)^{1/(1/2-a)}
			\left( \frac{\tau^2}{\sqrt{d}} \right)^{(1/2+a)/(1/2-a)}
			$
			\COMMENT{Scaling factor}
			%\STATE Draw a sample $S$ of $m$ examples $(\vec{x}_i, y_i)$ from $\exor$
			\STATE $\vec{z} \gets 
			\sum_S \frac{\vec{x}_i}{\|\vec{x}_i\|} \cdot y_i$
			\STATE 
			$R \gets_r \constructfoam(d)$ (Algorithm~\ref{alg:construct-foam})
			\COMMENT{Rounding scheme $R: \R^d \rightarrow \Z^d$}
%			\COMMENT{$R$ is constructed probabilistically by Algorithm~\ref{alg:construct-foam}, Algorithm~$1$ in \cite{KDRW12}) to construct  rounding scheme $R: \R^d \rightarrow \Z^d$ }
			\STATE $\vec{w} \gets R(k \cdot \vec{z})$ 
			\RETURN Hypothesis $h(\vec{x}) \eqdef
			\frac{\vec{x}}{\|\vec{x}\|} 
			\cdot \frac{ \vec{w} }{ \|\vec{w}\|}$
		\end{algorithmic}
	\end{algorithm}
\end{figure}

%In practice, lazy evaluation can be used to construct the rounding scheme $R$.

The subroutine $\constructfoam$ previously appeared as Algorithm~$1$ in \cite{KDRW12}. For completeness, we include a description below (Algorithm~\ref{alg:construct-foam}). 

\begin{figure}[H]
	\begin{algorithm}[H]
		\caption{$\constructfoam(d)$ 
			~ // Algorithm~1 in \cite{KDRW12}
			\\ Input: dimension $d$
			\\ Output: rounding scheme $R: \R^d \rightarrow \Z^d$
		}	
		\label{alg:construct-foam}
		\begin{algorithmic}
			\STATE Let $f: [0,1]^d \rightarrow \R$ s.t. $f(x_1, \dots, x_d) \eqdef \prod_{i=1}^d (2 \sin^2 (\pi x_i))$
			\STATE Let all points in $\R^d$ be unassigned
			\FOR {stage $t = 1, 2, \dots$ until all points are assigned}
				\STATE Uniformly at random sample $Z_t, H_t$ from $[0,1)^d \times (0, 2^d)$.
				\STATE Let droplet $D_i$ be the set of points $\{x | x \in -Z_t + [0,1)^d, f(x + Z_t) > H_t\}$.
				\STATE Let $R$ map all currently unassigned points in $D_i$ to $(0,0, \dots, 0)$ and extend this assignment periodically to all integer lattice points.
			\ENDFOR
			\RETURN $R$ 
		\end{algorithmic}
	\end{algorithm}
\end{figure}

%\subsection{Reproducible Halfspace Weak Learner}

%We assume that the true halfspace $f(\vec{x}) = \vec{w} \cdot \vec{x}$ has margin $\\amma$ in distribution $D$. 
%$ in distribution $D$. 
The following is the main result of this section. 

\begin{theorem}
	\label{thm:reproducible-halfspace-wkl}
	Let $D$ be a distribution over $\R^d$, and let $f: \R^d \rightarrow \pmone$ be a halfspace with margin $\tau$ in $D$. 
%	Let $\exor \eqdef \exor(D,f)$ be an example oracle. 
	Then $\rhalfspacewklfull$
	is a $(\rho, \tau/4, \rho/4)$-weak learner for halfspaces.
	That is,
	 Algorithm~\ref{alg:reproducible-halfspace} $\rho$-reproducibly returns a hypothesis
	$h$ such that,
	with probability at least $1-\rho/2$,
	$\frac{1}{2}\Exp_{\vec{x} \sim D} h(\vec{x}) f(\vec{x}) \ge \tau/4$, 
	using a sample of size $
	m
	=
	\left( \frac{896 \sqrt{d}}{\tau^2 \rho} \right)^{20/9}
	$. 
\end{theorem}

	%Perhaps add a picture depicting the above paragraph in $R^d$.	
\begin{proof}
	\textbf{Correctness (Advantage):} 
	We argue correctness in two parts. 	First, we show the expected weighted vector $\Exp_{\vec{x} \sim D} \left[\frac{\vec{x} \cdot f(\vec{x})}{\|\vec{x}\|} \right]$  defines a halfspace with good advantage (see Lemma~\ref{lem:adv-expected-weighted-vector-hypothesis}), following the arguments presented in Theorem~$3$ of \cite{Ser02PAC}. Then, we argue that rounding the empirical weighted vector $\vec{z}$ in Algorithm~\ref{alg:reproducible-halfspace} only slightly rotates the halfspace. By bounding the possible loss in advantage in terms of the amount of rotation (Lemma~\ref{lem:adv-perturbed-halfspace}), we argue that the rounded halfspace $\vec{w}/|\vec{w}|$ also has sizable advantage.

	By Lemma~\ref{lem:adv-expected-weighted-vector-hypothesis}, the expected weighted vector $\Exp_{\vec{x} \sim D} \left[\frac{\vec{x} \cdot f(\vec{x})}{\|\vec{x}\|} \right]$ has advantage $\tau/2$ on $D$ and $f$. 
	The martingale-based concentration bound in Corollary~\ref{cor:concentration-inequality} implies that the distance between $\vec{z}$ and 
	$\Exp[\vec z] = m \cdot \Exp_{\vec{x} \sim D} \left[\frac{\vec{x} \cdot f(\vec{x})}{\|\vec{x}\|} \right]$
	is less than $4m^{1/2+a}$ with probability at least $1-e^{-m^{2a}/2}$ for any $a \in (0,1/2)$ (chosen later). 
	Then, the vector is scaled by $k$ and rounded. 
	Any rounding scheme $R$ randomly generated by $\constructfoam$ always rounds its input to a point within distance $\sqrt{d}$ (Observation~\ref{obs:foam-rounding-max-distance}).
	Combining, the total distance between vectors $\frac{\vec{w}}{k \cdot \|\Exp[\vec z]\|}$ 
	and
	$\frac{\Exp[\vec z]}{\|\Exp[\vec z]\|}
	$
	%$\Exp_{\vec{x} \sim D} \left[\frac{\vec{x} \cdot f(\vec{x})}{\|\vec{x}\|} \right]$
	is at most
	$$
	\frac{4 m^{1/2 + a} + \sqrt{d} / k}{\|\Exp[\vec z]\|}
	\text{.}
	$$
	As $D$ has margin $\tau$ with respect to $f$, for all $\vec{x} \in \supp(D)$, $\frac{\vec{x} }{\|\vec{x}\|} \cdot f(\vec{x})$ has length at least $\tau$ in the direction of the expected weighted vector 	$\Exp_{\vec{x} \sim D} \left[\frac{\vec{x} \cdot f(\vec{x})}{\|\vec{x}\|} \right]$.
\iffalse
	As $D$ has margin $\tau$ with respect to $f$, each $\frac{\vec{x}_i }{\|\vec{x}_i\|} \cdot f(\vec{x}_i)$ term in the summation for $\vec{z}$ has length at least $\tau$ in the direction of the expected weighted vector 	$\Exp_{\vec{x} \sim D} \left[\frac{\vec{x} \cdot f(\vec{x})}{\|\vec{x}\|} \right]$.
\fi
	%Possible improvement: May be able to remove the $\tau$ parameter in the bound for norm of z (and improve sample complexity) by doing randomized rounding on a normalizing factor 
	Thus, 
%	$\|\vec z\| \ge \gamma m$, 
	$\|\Exp[\vec z]\| \ge \tau m$, 
	and the above quantity is at most
	$
	\frac{4m^{1/2 + a} + \sqrt{d} / k}{\tau m} 
	$.
	Simplifying, 
	$
	\frac{4m^{1/2 + a}}{\tau m} 
	= \frac{4}{\tau m^{1/2-a}}
	= \frac{4\tau}{896}\frac{\rho}{\sqrt{d}}
	< \tau/8
	$
	and
	$
	\frac{\sqrt{d} / k}{\tau m} 
	= \frac{\sqrt{d}}{\tau} \frac{\tau^2}{8\sqrt{d}}
	= \tau/8
	$.
	By applying Lemma~\ref{lem:adv-perturbed-halfspace} with $\theta = \tau/8 + \tau/8 = \tau/4$, we can conclude that $h$ has advantage at least $\tau/2 - \tau/4 = \tau/4$, 
	as 
	desired.\footnote{A 
		dedicated reader may notice that the scaling factor $k$ is subconstant. A possible error may arise if the scaling factor is so small that the halfspace vector $\vec{z} \cdot k$ gets rounded to $0$ by the rounding function $R$ (constructed by $\constructfoam$). Fortunately, with our choice of parameters, this turns out to not be an issue. The empirical vector sum $\vec{z}$ has norm at least $\tau \cdot m$, where $\tau$ is the margin size and $m$ is the sample complexity. As we have chosen scaling factor $k$ such that $m \cdot k = 8\sqrt{d}/\tau^2$, the input given to $R$ has norm at least $8\sqrt{d}/\tau$. Every rounding function $R$ constructed by $\constructfoam$ rounds its input to a point at distance at most $\sqrt{d}$ away (Observation~\ref{obs:foam-rounding-max-distance}), so we can be sure that $R$ never rounds our vector to the zero vector.}

	\textbf{Reproducibility:}
	Let $\vec{z}_1$ and $\vec{z}_2$ denote the empirical sums of vectors $\vec{x}_i y_i$ from two separate runs of $\rhalfspacewkl$. 
	It suffices to show that the rounding scheme $R$ constructed by $\constructfoam$ rounds 
	$k \cdot \vec{z}_1$ and $k \cdot \vec{z}_2$
	 to the same vector $\vec{w}$ with high probability. 
	The distance between $\vec{z}_1$ and $\vec{z}_2$ is at most $2 \cdot 4 m^{1/2+a}$ with probability at least 
	$1-2e^{-m^{2a}/2}
	$, 
	by
	 Corollary~\ref{cor:concentration-inequality}, 
	 the triangle inequality, and a union bound. 
	After scaling by $k$, this distance is at most
	$
	8 k m^{1/2+a}
	$.
	By Lemma~\ref{lem:foam-rounding}, the probability that $R$ does not round 
	$k \cdot \vec{z}_1$ and $k \cdot \vec{z}_2$
	to same integer lattice point is at most $
	7 \cdot 8 k m^{1/2+a}
	$.
	Altogether, the reproducibility parameter is at most
	$$
	2e^{-m^{2a}/2} 
	+ 	56 k m^{1/2+a}
	\text{.}
	$$
	The second term satisfies
	$
	56 k m^{1/2+a}
	= 
	\iffalse
	56 
	\cdot \left(
	8 \cdot 
	\left( \frac{\rho}{896} \right)^{1/(1/2-a)}
	\left( \frac{\tau^2}{\sqrt{d}} \right)^{(1/2+a)/(1/2-a)}
	\right)
	\cdot 
	\left( 
	\frac{896 \sqrt{d}}{\tau^2 \rho} \right)^{(1/2+a)/(1/2-a)}
	= 
	\fi
	448 \cdot \frac{\rho}{896} \cdot 1
	= \rho/2
	$,
	and the first term
	$
	2e^{-m^{2a}/2}
	\le \rho/2
	$
	when
	$
	m \ge \left( 2\ln(4/\rho) \right)^{1/(2a)}	
	$.
	So, as long as $a$ is chosen such that
	$m = 
	\left( \frac{896 \sqrt{d}}{\tau^2 \rho} \right)^{1/(1/2-a)}
	\ge 
	\left( 2\ln(4/\rho) \right)^{1/(2a)}	
	$,
	the algorithm is $\rho$-reproducible.
	This occurs 
	if 
	$
	\left( \frac{896}{\rho} \right)^{2a/(1/2-a)}
	\ge 
	 2\ln(4/\rho) 
	$, which is true for all values of $\rho \in (0,1)$ when $a = .05$.\footnote{The constant $a$ can be improved slightly if $a$ is chosen as a function of $\rho$.}
	
	%Could very slightly improve sample complexity parameter as follows: Make $a$ a function of $\rho$ for this last part of the argument; then we get a sample complexity with an exponent in terms of \rho, slightly better than 20/9
	
	\textbf{Failure rate:}	
	The algorithm succeeds when the martingale concentration bound holds. So, the failure probability of $\rhalfspacewkl$ is at most $e^{-m^{2a}/2} \le \rho/4$. 
	
	\textbf{Sample complexity:}
	Plugging in $a = .05$ in the expression
	$m = \left( \frac{896 \sqrt{d}}{\tau^2 \rho} \right)^{1/(1/2-a)}$ 
	yields the conclusion. 	
\end{proof}

\subsection{Reproducible Weak Halfspace Learner -- Definitions and Lemmas}

\subsubsection{Foams-Based Rounding Scheme from \texorpdfstring{\cite{KDRW12}}{[KDRW12]}}

For completeness, we restate relevant results from \cite{KDRW12} for our construction.

\begin{lemma}[Combining Theorem~$1$ and Theorem~$3$ of \cite{KDRW12}]
	\label{lem:foam-rounding}
	Let $R: \R^d \rightarrow \Z^d$ be the randomized rounding scheme constructed by Algorithm~\ref{alg:construct-foam} (Algorithm~1 in \cite{KDRW12}). Let $x, y \in \R^d$, and let  $\eps \eqdef d_{l_2} (x , y)$. Then $\Pr[R(x) = R(y)] \ge 1 - O(\eps)$, where the probability is over the randomness used in the algorithm.
\end{lemma}

\begin{proof}
	Theorem~$3$ of \cite{KDRW12} states that 
	$f(\vec{x}) = \Pi_{i=1}^d (2\sin^2(\pi x_i))$
	 is a \textit{proper} density function
	and $\int_{[0,1)^d} |\langle \nabla f, u \rangle| \le 2 \pi$ for all unit vectors $u$. 
	Theorem~$1$ of \cite{KDRW12} states the following. Let $f$ be a proper density function, and points $x,y \in \R^d$ such that $y = x + \eps \cdot u$, where $\eps > 0$ and $u$ is a unit vector. Let $N$ denote the number of times the line segment $\overline{xy}$ crosses the boundary between different droplets (potentially mapping to the same integer lattice point) in an execution of $\constructfoam$. Then
	$\Exp[N] \approx \eps \cdot \int_{[0,1)^d]} |\langle \nabla f, u \rangle|$, where the $\approx$ notation is hiding a $W \eps^2$ term, where $W>0$ is a universal constant depending only on $f$. 
	The authors refine this statement (\cite{KDRW12}, page 24) to show that the $W \eps^2$ term can be made arbitrarily small. 
	Combining, 
	$\Exp[N] \le 2\pi\eps  + W\eps^2 < 6.3  \eps$.
	By Markov's inequality, $\Pr[N = 0]	< 1 - 6.3 \eps$. 
\end{proof}

\begin{observation}
	\label{obs:foam-rounding-max-distance}
	$\constructfoam$ always outputs a rounding scheme $R$ with the following property: the ($l_2$) distance between any vector $\vec{v} \in \R^d$ and $R(\vec{v})$ is at most $\sqrt{d}$.
\end{observation}
This follows from noticing that $R$ maps each coordinate of $\vec{v}$ to its floor or ceiling.

\begin{theorem}[Runtime of $\constructfoam$; \cite{KDRW12}, page 25]
	\label{thm:foam-construction-runtime}
	There are universal constants $1 < c < C$ such that Algorithm~\ref{alg:construct-foam}, when run with
	$f(\vec{x}) = \Pi_{i=1}^d (2\sin^2(\pi x_i))$, 
	takes between $c^d$ and $C^d$ stages except with probability at most $c^{-d}$.
\end{theorem}

\subsubsection{Weak Learning Definitions}
% change titles and ordering of results later

\begin{definition}[Weak Learning Algorithm (in the Filtering Model)]
	Let $\CCC$ be a concept class of functions from domain $\data$ to $\pmone$, and let $f \in \CCC$. Let $D$ be a distribution over $\data$. 
	Let $\wkl$ be an algorithm that takes as input a labeled sample $S = \{(x_i, f(x_i))\}_m$ drawn i.i.d. from $D$, and outputs a hypothesis $h: \data \rightarrow [-1,1]$. 	
	Then $\wkl$ is a $(\gamma, \delta)$-\textit{weak learner} for $\CCC$ with sample complexity $m$ if, for all $f, D$, with probability at least $1 - \delta$, $\wkl(S)$ outputs a hypothesis $h: \data \rightarrow [-1,1]$ such that
	$\Exp_{x \sim D} f(x) h(x) \ge 2\gamma$,
	where $S$ is a sample of size $|S| = m$ drawn i.i.d. from $D$.
\end{definition}

We say a $(\gamma, \delta)$-weak learner has \textit{advantage} $\gamma$.
Equivalently, if a hypothesis $h$ satisfies
$\frac{1}{2}\Exp_{x \sim D} f(x) h(x) \ge \gamma$, then we say $h$ has advantage $\gamma$ (on $D$ and $f$). 

\begin{definition}[Reproducible Weak Learning Algorithm]
	Algorithm $\rwkl$ is a $(\rho, \gamma, \delta)$-weak learner if $\rwkl$ is $\rho$-reproducible and a $(\gamma, \delta)$-weak learner. 
\end{definition}

\subsubsection{Halfspaces and Their Advantage}

\margindef*

\begin{lemma}[Advantage of Expected Weighted Vector Hypothesis \cite{Ser02PAC}]
	\label{lem:adv-expected-weighted-vector-hypothesis}
	Let $f(\vec{x}) \eqdef \sgn(\vec{w} \cdot \vec{x})$ be a halfspace, and let $D$ be a distribution over $\R^d$ with margin $\tau$ with respect to $f$. Let $\vec{z} 
	= \Exp_{\vec{v} \sim D} \left[ \frac{\vec{v}}{\|\vec{v}\|} f(\vec{v}) \right]$. 
	Then the hypothesis 
	$h_{\vec{z}}(\vec{x}) 
	= \frac{\vec{x}}{\|\vec{x}\|} \cdot \frac{\vec{z}}{\|\vec{z}\|}
	$
	has advantage at least $\tau/2$. 
\end{lemma}

\begin{proof}

%	By the margin assumption, $h_{\vec{z}} (\vec{x}) f(\vec{x}) \ge \tau$ for all $\vec{x} \in \supp(D)$. 

	The advantage of $\hz$
	is
	$\frac{1}{2} \Exp_{\vec{x} \sim D} [\hz(\vec{x}) f(\vec{x})]
	= \frac{1}{2} \Exp_{\vec{x} \sim D} \left[ \frac{\vec{x}}{\|\vec{x}\|} \cdot \frac{\vec{z}}{\|\vec{z}\|} \cdot f(\vec{x}) \right]
	=  \frac{ \vec{z} \cdot \vec{z}}{2\|\vec{z}\|} 
	= \frac{\|\vec{z}\|}{2}
	\ge 
	\frac{\vec{z} \cdot \vec{w}}{2\|\vec{w}\|} 
	$,
	by the Cauchy-Schwarz inequality.
	Vector $\vec{z}$ is a convex combination of 
	$\frac{\vec{x} \cdot f(\vec{x}) }{\|\vec{x}\|}$ terms, for $\vec{x} \in \supp(D)$.
	By the margin assumption, 
	$
	\frac{\vec{x} \cdot f(\vec{x}) }{\|\vec{x}\|} 
	\cdot \frac{ \vec{w} }{ \|\vec{w}\|} 
	\ge \tau$ for all $x \in \supp(D)$. 
	Thus, 
	$\frac{\vec{z} \cdot \vec{w}}{2\|\vec{w}\|} \ge \frac{\tau}{2}$.

\end{proof}

\iffalse
%Note: I don't think the following corollary is very helpful
Notice that the $\tau/2$ advantage guarantee of Lemma~\ref{lem:adv-expected-weighted-vector-hypothesis} holds for any reweighting of distribution $D$.
Formally, for any $\mu: \data \rightarrow [0,1]$, the distribution $D'$ such that 
$
D'(\vec{x}) =  \frac{ \mu(\vec{x}) \cdot D(\vec{x})}{\Exp [\mu(\vec{x}) \cdot D(\vec{x})] }
$
is also a distribution with margin $\tau$ with respect to $f$. 
In particular, any sample $S \sim D$ naturally induces a reweighting of $D$, in which $\mu(\vec{x}_i)$ is proportinal to the number of times $\vec{x}_i$ appears in $S$. 

\begin{corollary}[Advantage of Weighted Vector Hypothesis]
	%\label{lem:adv-weighted-vector-hypothesis}
	Let $f(\vec{x}) \eqdef \sgn(\vec{w} \cdot \vec{x})$ be a halfspace, and let sample $S$ be drawn from a distribution $D$  over $\R^d$ with margin $\tau$ with respect to $f$. Let $\vec{z} 
	= \Exp_{\vec{v} \in S} \left[ \frac{\vec{v}}{\|\vec{v}\|} f(\vec{v}) \right]$. 
	Then the hypothesis 
	$h_{\vec{z}}(\vec{x}) 
	= \frac{\vec{x}}{\|\vec{x}\|} \cdot \frac{\vec{z}}{\|\vec{z}\|}
	$
	has advantage at least $\tau/2$. 
\end{corollary}
\fi

\begin{lemma}[Advantage of Perturbed Halfspaces]
	\label{lem:adv-perturbed-halfspace}
	Consider a halfspace defined by unit vector $\vec{w}$, and let 
	$h(\vec{x}) = 
	\frac{\vec{x}}{\|\vec{x}\|} 
	\cdot  \vec{w}$.
	Assume $h$ has advantage $\gamma$, i.e.
	$\frac{1}{2}\Exp_{x \sim D} f(x) h(x) \ge \gamma$.  
	Let $\vec{u}$ be any vector such that $\|\vec{u}\| \le \theta$, where $\theta \in [0,\sqrt{3}/2)$. 
	Let perturbed vector $\vec{w'} = 
	\frac{ \vec{w} + \vec{u}}
	{\| \vec{w} + \vec{u}\|}
	$, and let
	 $h'(\vec{x}) = 
	 \frac{\vec{x}}{\|\vec{x}\|} 
	 \cdot  \vec{w'}$.
	Then $h'$ has advantage at least $\gamma - \theta$. 
\end{lemma}

%Perhaps add a picture to better explain this proof
\begin{proof}
	First, we bound the maximum distance between $\vec{w}$ and $\vec{w'}$. Then, we apply Cauchy-Schwarz to bound the advantage loss. 
	$\vec{w'}$ is constructed by perturbing $\vec{w}$ by a vector $\vec{u}$, and then normalizing to norm $1$. $\vec{w'}$ is furthest away from $\vec{w}$ when the vector $\vec{w'}$ is tangent to the ball of radius $\|\vec{u}\|$ around $\vec{w}$. In this case, 
	$\|\vec{w'} - \vec{w}\|^2 
	= (1-\sqrt{1-\theta^2})^2 + \theta^2 
	= 2 - 2\sqrt{1-\theta^2}$.
	Since $\theta < \sqrt{3}/2$, $2 - 2\sqrt{1-\theta^2} < 4\theta^2$. So, 
	$\|\vec{w'} - \vec{w}\|^2 < 4 \theta^2
	$.
	The advantage of $h'$ is
	\begin{equation*}
		\begin{split}
			\frac{1}{2}\Exp_{\vec{x} \sim D} 
			\left[
			\frac{\vec{x}}{\|\vec{x}\|} 
			\cdot  \vec{w'}
			\cdot
			f(\vec{x})
			\right] 
			&= 
			\frac{1}{2}
			\Exp_{\vec{x} \sim D} 
			\left[
			\frac{\vec{x}}{\|\vec{x}\|} 
			\cdot (\vec{w} + (\vec{w'} - \vec{w}))
			\cdot
			f(\vec{x})
			\right]
			\\&= 	
			\gamma +
			\frac{1}{2}
			(\vec{w'} - \vec{w}) \cdot 
			\Exp_{\vec{x} \sim D} 
			\left[
			\frac{\vec{x}}{\|\vec{x}\|} 
			\cdot
			f(\vec{x})
			\right]
		\end{split}
	\end{equation*}

By Cauchy-Schwarz, the second term of the right-hand side has magnitude at most $\sqrt{4\theta^2 \cdot 1}/2$, so the advantage of $h'$ is at least $\gamma - \theta$. 
\end{proof}

\subsubsection{Concentration Bound on Sum of Normalized Vectors}

Let $D$ be a distribution on $\R^n$.
Let ${\bf v} = \{ {\bf v_1}, \ldots, {\bf v_T}\} \in D^T$
be a random sample of $T$ vectors from $D$ with the
following properties:
\begin{enumerate}
	\item $\Exp_{{\bf v} \in D^T} [ \sum_{i=1}^{T} {\bf v_i}] - 
	\Exp_{v \in D} [v]  =0$.
	\item $\forall v \in D$, $|| v ||_2 \leq c$.
\end{enumerate}

\begin{restatable}{lemma}{concentrationinequalitylem}
	\label{lem:concentration-inequality}
	Let $D, {\bf v} \in D^T$ satisfy properties
	(1) and (2) above, and
	let ${\bf v^{\leq T}} = \sum_{i=1}^T {\bf v_i}$.
	Then for all $\Delta>0$, 
	%$0 \leq \Delta \leq T-{\sqrt T}$,
	$$\Pr_{{\bf v}} [|| {\bf v^{\leq T}} ||_2 \geq \sqrt{T}(1 + c/2) + \Delta] \leq e^{-\Delta^2/2c^2T}.$$
\end{restatable}

For a proof, see Appendix~\ref{apps:concentration-inequality}.

\begin{corollary}
	\label{cor:concentration-inequality}
	Let $D$ be a distribution supported on the unit ball in $d$ dimensions, and let $f$ be a halfspace. 
	Let $S$ be a sample of $T$ examples $(\vec{x}_i, f(\vec{x}_i))$ drawn i.i.d. from $D$, and let $\vec{z} = \sum_S \vec{x}_i \cdot f(\vec{x}_i)$. 
	Let $a \in (0,1/2)$. 
	Then
	$\Pr_{S \sim D} 
	\left[ 
	\|\vec{z} - 
	T \Exp_{\vec{v} \sim D} [\vec{v} f(\vec{v})] \|
	\ge 4 T^{1/2 + a} 
	\right]
	\le e^{-T^{2a}/2}
	$.
\end{corollary}

\begin{proof}
	In order to have $D$ satisfy the properties (1) and (2) above, we must translate $D$ by the expectation $\Exp_{\vec{v} \sim D} [\vec{v}f(\vec{v})]$. After this translation, the maximum length of a vector in the support is $c = 2$. Plugging in $\Delta = 2T^{1/2 + a}$ and noting $2T^{1/2 + a} \ge 2T^{1/2}$ yields the conclusion. 
\end{proof}

\subsection{Coordinate-Based Rounding Scheme}

Algorithm $\rhalfspace$ uses polynomial sample complexity and runs in polynomial time except for subroutine $\constructfoam$, which runs in expected exponential time in the dimension $d$ (Theorem~\ref{thm:foam-construction-runtime}). 
Next, we consider a simpler rounding scheme that rounds points coordinate-by-coordinate to a randomly shifted integer lattice. 
This rounding scheme requires tighter concentration bounds, resulting in approximately another factor of $d$ in the sample complexity. In return, it can be constructed by $\constructboxes$ and executed in linear time in sample complexity $m$ and dimension $d$.

\begin{figure}[H]
	\begin{algorithm}[H]
		\caption{$\constructboxes(d)$ 
			~ // constructs coordinate-based rounding schemes
			\\ Input: dimension $d$
			\\ Output: rounding scheme $R: \R^d \rightarrow \R^d$
		}	
		\label{alg:construct-boxes}
		\begin{algorithmic}
			\STATE Uniformly at random draw $Z$ from $[0,1)^d$.
			\STATE Let box $B$ be the set of points $\{x | \forall i \in [d], x_i \in [-1/2 + z_i, 1/2 + z_i)\}$
			\STATE Let $R$ map all points in $B$ to point $Z$ and extend this assignment periodically by integer lattice points
			\RETURN $R$ 
		\end{algorithmic}
	\end{algorithm}
\end{figure}

\begin{lemma}
	\label{lem:boxes-rounding}
	Let $R: \R^d \rightarrow \R^d$ be the randomized rounding scheme constructed by $\constructboxes$. Let $\vec x, \vec y \in \R^d$, and let  $\eps \eqdef d_{l_2} (\vec x , \vec y)$. 
	Then $\Pr[R( \vec x) = R( \vec y)] \ge 1 - d \eps $.
\end{lemma}

\begin{proof}
	We bound this probability by a crude $l_2$ to $l_1$ distance conversion. 
	If $\vec x$ and $\vec y$ have $l_2$ distance $\eps$, then the distance between $x_i$ and $y_i$ is at most $\eps$ for all coordinates $i \in [d]$. The $i$'th coordinate of $\vec x$ and $\vec y$ are not rounded to the same point with probability $|x_i - y_i|$. By a union bound,
	$\Pr[R( \vec x) = R( \vec y)] \ge 1 - d \eps$.
\end{proof}

\begin{observation}
	\label{obs:boxes-rounding-max-distance}
	$\constructboxes$ always outputs a rounding scheme $R$ with the following property: the ($l_2$) distance between any vector $\vec{v} \in \R^d$ and $R(\vec{v})$ is at most $\sqrt{d}/2$.
\end{observation}
This follows from noticing that $R$ maps each coordinate of $\vec{v}$ to value within distance $1/2$.

\subsubsection{Reproducible Halfspace Weak Learner using Boxes}

\begin{figure}[H]
	\begin{algorithm}[H]
		\caption{$\rhalfspacewklboxesfull$ 
			~ // a $\rho$-reproducible halfspace weak learner
			\\Parameters: desired reproducibility $\rho$, dimension $d$, assumed margin $\tau$, constant $a = .1$
			\\Input: A sample $S$ of 
			$m = 
			 \left( \frac{64 d^{3/2}}{\tau^2 \rho} \right)^{1/(1/2-a)}
			$
			examples $(\vec{x}_i, y_i)$ drawn i.i.d. from distribution $D$
			\\Output: A hypothesis with advantage $\gamma/4$ on $D$ against $f$
		}	
		\label{alg:reproducible-halfspace-wkl-boxes}
		\begin{algorithmic}
			\STATE $k \gets \frac{1}{m} \frac{4\sqrt{d}}{\tau^2}
			= 4 \cdot 
			\left( \frac{\rho \cdot \tau ^{1+2a}}{64 \cdot d^{5/4 + a/2}} \right)^{1/(1/2-a)}
			$
			\COMMENT{Scaling factor}
			\STATE $\vec{z} \gets 
			\sum_S \frac{\vec{x}_i}{\|\vec{x}_i\|} \cdot y_i$
			\STATE 
			$R \gets_r \constructboxes(d)$ (Algorithm~\ref{alg:construct-boxes})
			\COMMENT{Rounding scheme $R: \R^d \rightarrow \R^d$}
			\STATE $\vec{w} \gets R(k \cdot \vec{z})$ 
			\RETURN Hypothesis $h(\vec{x}) \eqdef
			\frac{\vec{x}}{\|\vec{x}\|} 
			\cdot \frac{ \vec{w} }{ \|\vec{w}\|}$
		\end{algorithmic}
	\end{algorithm}
\end{figure}

\begin{theorem}
	\label{thm:reproducible-halfspace-wkl-boxes}
	Let $D$ be a distribution over $\R^d$, and let $f: \R^d \rightarrow \pmone$ be a halfspace with margin $\tau$ in $D$. 
	%	Let $\exor \eqdef \exor(D,f)$ be an example oracle. 
	Then $\rhalfspacewklfull$
	is a $(\rho, \tau/4, \rho/4)$-weak learner for halfspaces.
	That is,
	Algorithm~\ref{alg:reproducible-halfspace} $\rho$-reproducibly returns a hypothesis
	$h$ such that,
	with probability at least $1-\rho/2$,
	$\Pr_{\vec{x} \sim D} h(\vec{x}) f(\vec{x}) \ge \tau/4$, 
	using a sample of size $
	m
	=
	\left( \frac{64 d^{3/2}}{\tau^2 \rho} \right)^{5/2}
	$. 
\end{theorem}

\begin{proof}
	\textbf{Correctness (Advantage):} 
	The proof proceeds almost identically to the proof of Theorem~\ref{thm:reproducible-halfspace-wkl}.
	By Lemma~\ref{lem:adv-expected-weighted-vector-hypothesis}, the expected weighted vector $\Exp_{\vec{x} \sim D} \left[\frac{\vec{x} \cdot f(\vec{x})}{\|\vec{x}\|} \right]$ has advantage $\tau/2$ on $D$ and $f$. 
	Any rounding scheme $R$ randomly generated by $\constructboxes$ always rounds its input to a point within distance $\sqrt{d}/2$, so the distance between vectors $\frac{\vec{w}}{k \cdot \|\Exp[\vec z]\|}$ 
	and
	$\frac{\Exp[\vec z]}{\|\Exp[\vec z]\|}$
		is at most
	$$
	\frac{4 m^{1/2 + a} + \sqrt{d} / 2k}{\tau m}
	\text{.}
	$$
	
	Simplifying, 
	$
	\frac{4m^{1/2 + a}}{\tau m} 
	= \frac{4}{\tau m^{1/2-a}}
	= \frac{4\tau}{64}\frac{\rho}{d^{3/2}}
	< \tau/8
	$
	and
	$
	\frac{\sqrt{d} / 2k}{\tau m} 
		= \frac{\sqrt{d}}{2\tau} \frac{\tau^2}{4\sqrt{d}}
	= \tau/8
	$.
	By applying Lemma~\ref{lem:adv-perturbed-halfspace} with $\theta = \tau/8 + \tau/8$, we can conclude that $h$ has advantage at least $\tau/2 - (\tau/8 + \tau/8) = \tau/4$, 
	as 
	desired.
	
	\textbf{Reproducibility:}
	Let $\vec{z}_1$ and $\vec{z}_2$ denote the empirical sums of vectors $\vec{x}_i y_i$ from two separate runs of $\rhalfspacewkl$. 
	It suffices to show that the rounding scheme $R$ constructed by $\constructboxes$ rounds 
	$k \cdot \vec{z}_1$ and $k \cdot \vec{z}_2$
	to the same vector $\vec{w}$ with high probability. 
	The distance between $\vec{z}_1$ and $\vec{z}_2$ is at most $2 \cdot 4 m^{1/2+a}$ with probability at least 
	$1-2e^{-m^{2a}/2}
	$, 
	by
	Corollary~\ref{cor:concentration-inequality}, 
	the triangle inequality, and a union bound. 
	After scaling by $k$, this distance is at most
	$
	8 k m^{1/2+a}
	$.
	By Lemma~\ref{lem:boxes-rounding}, the probability that $R$ does not round 
	$k \cdot \vec{z}_1$ and $k \cdot \vec{z}_2$
	to same integer lattice point is at most $
	d \cdot 8 k m^{1/2+a}
	$.
	Altogether, the reproducibility parameter is at most
	$$
	2e^{-m^{2a}/2} 
	+ 	8 dk m^{1/2+a}
	\text{.}
	$$
	The second term satisfies
	$
	8 d k m^{1/2+a}
	= 
	8d (km/ m^{1/2-a})
	= \rho/2
	$,
	and the first term
	$
	2e^{-m^{2a}/2}
	\le \rho/2
	$
	when
	$
	m \ge \left( 2\ln(4/\rho) \right)^{1/(2a)}	
	$.
	So, as long as $a$ is chosen such that
	$m = 
	\left( \frac{64 d^{3/2}}{\tau^2 \rho} \right)^{1/(1/2-a)}
	\ge 
	\left( 2\ln(4/\rho) \right)^{1/(2a)}	
	$,
	the algorithm is $\rho$-reproducible.
	This occurs 
	if 
	$
	\left( \frac{64}{\rho} \right)^{2a/(1/2-a)}
	\ge 
	2\ln(4/\rho) 
	$, which is true for all values of $\rho \in (0,1)$ when $a = .07$.
	For simpler constants, we use $a = .1$. 
	
	\textbf{Failure rate:}	
	The algorithm succeeds when the martingale concentration bound holds. So, the failure probability of $\rhalfspacewkl$ is at most $e^{-m^{2a}/2} \le \rho/4$. 
	
	\textbf{Sample complexity:}
	Plugging in $a = .1$ in the expression
	$m = \left( \frac{64 d^{3/2}}{\tau^2 \rho} \right)^{1/(1/2-a)}$ 
	yields the conclusion. 	
\end{proof}

\iffalse
\subsection{Discussion}

remarks about using a different rounding scheme (exponential time + better sample complexity (foams) vs. (coordinate-based random rounding) poly time and much worse sample complexity)

In $\R^d$, the foam-based rounding scheme from \cite{KDRW12}
 takes expected time $\exp(d)$ to construct. %A polynomial-time procedure with similar rounding guarantees would directly yield a polynomial-time reproducible halfspace learner. 

More broadly, any randomized rounding scheme that provides good guarantees for rounding ``nearby" objects can be combined with any concentration bound for ``nearby" objects, provided that the definitions of ``nearby" are compatible. 
For example, \cite{KDRW12} provides a direct connection between the notions of $l_2$-stability (e.g. \cite{Luxburg2010}) and reproducibility.
In the case of learning halfspaces, it would be interesting to combine the standard PAC-learning algorithm for halfspaces with a randomized rounding scheme on a more appropriate space than $\R^d$ (such as the unit ball in $d$ dimensions). 

maybe talk about how we don't have to fix $a$ in $\rhalfspacewkl$, but we can have $a$ depend on $\rho$ for slightly better outcomes in sample complexity

Locality-sensative hashing as a rounding scheme?
\fi 
\section{Reproducible Boosting}
\label{sec:boosting}

In this section, we argue that a small modification of the boosting algorithm in \cite{Servedio:03jmlr} is a reproducible boosting algorithm. 
Given access to a reproducible weak learner, this boosting algorithm $\rho$-reproducibly outputs a hypothesis. 
Boosting algorithms are a natural candidate for constructing reproducible algorithms ---
many boosting algorithms in the standard PAC-setting are deterministic, and the final classifier returned is often a simple function of the weak learner hypotheses (e.g. a majority vote).
Combining this reproducible boosting algorithm with our reproducible halfspace weak learner from Section~\ref{sec:halfspaces} yields a reproducible strong learner for halfspaces. 

Specifically, we modify the smooth boosting algorithm described in \cite{Servedio:03jmlr} in the batch setting, presenting it in the filtering setting \cite{BradleySchapire:07}. 
This boosting algorithm has three main components, all of which can be made reproducible: (i) checking for termination (via a statistical query), (ii) running the weak learner (reproducible by assumption), and (iii) updating the weighting function (deterministic). The final classifier is a sum of returned weak learner hypotheses. With high probability over two runs, our boosting algorithm $\rboost$ collects the exact same hypotheses $h_1, \dots, h_T$ from its reproducible weak learner.

%why we chose this boosting algo: bc it's pretty simple and straightforward, and the analysis immediately carries over to the filtering setting, which seems to be (a priori) easier to analyze than the batch setting, where we have to worry about data reuse

\subsection{Reproducible Boosting Algorithm: An Overview}

In smooth boosting algorithms, a ``measure" function $\mu: \data \rightarrow [0,1]$ determines a reweighting of distribution $D$. The induced reweighted distribution, denoted $D_\mu$, is defined by the probability density function $D_\mu (x) = \mu(x) \cdot D(x)/ d(\mu)$, where $d(\mu)$ is a normalizing factor $\Exp_{x \sim D} \mu(x)$. We refer to $d(\mu)$ as the \textit{density} of measure $\mu$. 
A sample $\vec s$ is drawn from $D_\mu$ and passed to the weak learner $\rwkl$. 
Sampling from $D_\mu$ using example oracle $\exor$ is done by rejection sampling --- draw a sample $(x,y)$ from $\exor$ and a random $b \in_r [0,1]$; if $r \le \mu(x)$, keep $(x,y)$; otherwise, reject $x$ and loop until we keep $(x,y)$. On expectation, we require $m/d(\mu)$ examples from $D$ to sample $m$ examples from $D_\mu$.

At the beginning of the algorithm, $\mu(x) = 1$ for all $x \in \supp(D)$. Weak learner hypotheses $h_t$ are used to modify update $\mu$ (and thus $D_\mu$) for future weak learner queries. 
The algorithm terminates when the density $d(\mu)$ drops below the desired accuracy parameter $\eps$ --- at this point, the majority vote hypothesis $\boldh = \sign(\sum_t h_t)$ has accuracy at least $1-\eps$ over $D$. 

More specifically, we define $\mu_{t+1}(x) = M(g_t(x))$ using a base measure function $M: \R \rightarrow [0,1]$ and score function $g: \data \rightarrow \R$. 
As in \cite{Servedio:03jmlr}, we 
use a capped exponential function as our base measure function 
$M(a) = \begin{cases}
1 & a \le 0 \\
(1 - \gamma)^{a/2} & a > 0
\end{cases}$. 
The score function is 
$g_t(x) = \sum_{i=1}^t (h_i(x) f(x) - \theta)$, where $\theta < \gamma$ is chosen as a function of $\gamma$.

\begin{figure}[H]
	\begin{algorithm}[H]
		\caption{$\rboostfull$ 
			~ // a $\rho$-reproducible boosting algorithm
			\\Input: 
				A sample $\vec s$ of $m$
					examples $(\vec{x}_i, y_i)$ drawn i.i.d. from distribution $D$.
			\\Access to reproducible weak learner $\rwkl$ with advantage $\gamma$ and sample complexity $m_\rwkl$.
%			Access to example oracle $\exor(D,f)$.		
			\\Parameters: desired reproducibility $\rho$, accuracy $\eps$, constant $\theta \eqdef \gamma/(2 + \gamma)$, 
			round complexity $T = O(1/\eps \gamma^2)$
			\\Output: A hypothesis 
			$\boldh = \sgn \left( \sum_{t=1}^T h_t \right)$, where the $h_t$'s are weak learner hypotheses. 
		}	
		\label{alg:reproducible-boosting}
		\begin{algorithmic}
			\STATE $g_0 (x) \eqdef 0$
			\STATE $\mu_1 (x) \eqdef M(g_0) = 1$
			\COMMENT{``Measure" function for reweighting}
			\STATE $t \gets 0$
			\WHILE{1}
				\STATE $t \gets t+1$
				\STATE $D_{\mu_t}(x) \eqdef \mu_t(x) \cdot D(x)/d(\mu_{t})$
				\COMMENT{Reweighted distribution}	
				\STATE $\vec s_1 \gets \widetilde O( {m_\rwkl}/{\eps} )$ fresh examples from $\vec s$			
				\STATE $\vec s_\rwkl \gets \rejectionsampler(\vec s_1, 
				m_\rwkl	, 
				\mu_t; r_1)$
				\COMMENT{Rejection sampling for $\rwkl$}
				\STATE Hypothesis $h_t \gets \rwkl(\vec s_\rwkl; r_2)$
				\STATE $g_t (x) \eqdef g_{t-1}(x) + h_t(x) f(x) -\theta$
				\COMMENT{Reweight distribution using $h_t$}
				\STATE $\mu_{t+1}(x) \eqdef M(g_t(x))$
				\STATE $\vec s_2 \gets \widetilde O 
				\left( 
				\frac{1}{\rho^2 \eps^3 \gamma^2}	
				\right)
				$ fresh examples from $\vec s$
				\COMMENT{{Run $\rstat$ to check if $d(\mu_{t+1}) \le \eps$}}
				\IF[tolerance $\tau = \eps/3$, reproducibility $\rho_0 = \rho/(3T)$]{$\rstat_{\tau, \rho_0, \phi}(\vec s_2; r_3) \le 2\eps/3$ }
%				//{Check if $d(\mu_t) \le \eps$}
				\STATE Exit while loop
				\COMMENT{failure rate $\rho/(12T)$, query $\phi(x,y) = \mu(x)$}
				\ENDIF
			\ENDWHILE
			\RETURN $\boldh \gets \sgn \left( \sum_{t} h_t \right)$
		\end{algorithmic}
	\end{algorithm}
\end{figure}

\begin{figure}[H]
	\begin{algorithm}[H]
		\caption{$\rejectionsamplerfull$ 
			~ // draw a sample from distribution $D_\mu$
			\\Input: sample $\vec s_{\text{all}}$ drawn i.i.d. from distribution $D$, target size of output sample $\mtarget \in [|\sall|]$, and 
			description of measure function $\mu: \data \rightarrow [0,1]$.		
%			\\Parameters: desired reproducibility $\rho$, desired number of samples $\mest$ kept in rejection sampling
			\\Output: $\bot$  or a sample $\skept$ of size $|\skept| = \mtarget$
		}	
		\label{alg:rejection-sampler}
		\begin{algorithmic}
			\STATE $\skept \gets \emptyset$
			\FOR{$i = 1$ to $i = |\sall|$}
				\STATE Use randomness $r$ to randomly pick a $b \in [0,1]$
				\IF[Reject $(x_i, y_i)$ w. p. $1 - \mu(x)$]{$\mu(x_i) \ge b$}
					\STATE $\skept \gets \skept || (x_i, y_i)$
					\COMMENT{Add example $(x_i, y_i)$ to $\skept$}
				\ENDIF
				\IF{$|\skept| = \mtarget$}
					\RETURN $\skept$
				\ENDIF
			\ENDFOR
			\RETURN $\bot$
			\COMMENT{Ran out of fresh samples in $\sall$}
		\end{algorithmic}
	\end{algorithm}
\end{figure}

A subtle note is that this boosting algorithm must precisely manage its sample $\vec s$ and random string $r$ when invoking subroutines. 
In order to utilize the reproducibility of subroutines (e.g. $\rwkl$), the boosting algorithm needs to ensure that it uses random bits from the same position in $r$.
A first-come first-serve approach to managing $r$ (i.e. each subroutine uses only the amount of randomness it needs) fails immediately for $\rboost$ --- the amount of randomness $\rejectionsampler$ needs is dependent on the sample, so the next subroutine (in this case, $\rwkl$) may not be using the same randomness across two (same-randomness $r$) runs of $\rboost$. 

If one can precisely upper bound the amount of randomness needed for each of $L$ subroutines, then $r$ can be split into chunks $r_1 || r_2 ||  \dots || r_L$, avoiding any desynchronization issues. Alternatively, one can split $r$ into $L$ equally long random strings by only using bits in positions equivalent to $l \mod L$ for subroutine $l \in [L]$.

\subsection{Analysis of {\texorpdfstring{$\rboost$}{rBoost}} (Algorithm~\ref{alg:reproducible-boosting})}

As before, function $f$ in concept class $C$ is a function from domain $\data$ to $\pmone$. $D$ is a distribution over $\data$. 

\begin{theorem}[Reproducible Boosting]
	\label{thm:reproducible-boosting}
	Let $\eps > 0, \rho > 0$.
	%, and $\delta > 0$. 
	Let $\rwkl$ be a $(\rho_r, \gamma, \delta_\rwkl)$-weak learner. 
	Then $\rboostfull$ is $\rho$-reproducible and 
	with probability at least $1-\rho$, 
	%with probability at least $1-\delta$, 
	outputs a hypothesis $\boldh$ such that
	$\Pr_{x \sim D} [\boldh(x) = f(x)] \ge 1- \eps$. 
	$\rboost$ runs for $T = O(1/(\eps \gamma^2_\rwkl))$ rounds and uses  
	$	\widetilde O 
	\left(
	\frac{	m_{\rwkl(\rho/(6T))}}{\eps^2 \gamma^2}
	%\footnote{Here, $m_\rwkl$ is a function and $\rho/(6T)$ is an input, not a multiplicative factor}
	+ 
	\frac{1}{\rho^2 \eps^3 \gamma^2}	
	\right)
	$ samples,
	where the $\widetilde{O}$ notation hides $\log(1/(\rho \eps \gamma^2))$ factors
	and $m_{\rwkl(\rho/(6T))}$ denotes the sample complexity of $\rwkl$ with reproducibility parameter $\rho/(6T)$. 
\end{theorem}

%The arguments for round complexity and correctness are very similar to those in \cite{Servedio:03jmlr}.

For readability, we break the proof into components for round complexity, correctness, reproducibility, sample complexity, and failure probability. 

\begin{proof}
	\textbf{Round Complexity: }
	Theorem~3 in \cite{Servedio:03jmlr} gives a $T = O(1/(\eps \gamma^2_\rwkl))$ round complexity bound for this boosting algorithm in the batch setting. Analogous arguments hold in the filtering setting, so we defer to \cite{Servedio:03jmlr} for brevity. 
	%also note that a potential-function argument can be used to show convergence, when the potnetial is defined as the integral of the measure function from neg infinity to the score $g(x) = sum$ of hypotheses so far
	
	\textbf{Correctness: }
	Similarly, since reproducible weak learner $\rwkl$ satisfies the definitions of a weak learner, the correctness arguments in \cite{Servedio:03jmlr} also hold. A small difference is the termination condition --- rather than terminate when the measure satisfies $d(\mu) < \eps$, our algorithm terminates when the density estimated by $\rstat$ is less than $2\eps/3$. 
	We run $\rstat$ on query $\phi(x) = \mu(x)$ with tolerance parameter $\eps/3$. Thus, when the $\rboost$ terminates, $d(\mu_t) < \eps$.

	\textbf{Reproducibility: }
	We show this boosting algorithm not only reproducibly outputs the same hypothesis $\boldh$, but that each returned weak learner hypothesis $h_t$ is identical across two runs of the boosting algorithms (using the same randomness $r$) with high probability. 
	The reweighted distribution $D_{\mu_t}$ depends only on the previous weak learner hypotheses $h_1, \dots, h_{t-1}$, so the only possibilities for loss of reproducibility are: (i) returning $\bot$ while rejection sampling from $D_\mu$; (ii) running the reproducible weak learner; and (iii) using a statistical query to decide to exit the while 
	loop. 
	We note that our choice of parameters adds non-reproducibilty at most $\rho/(3T)$ for each and apply a union bound over at most $T$ rounds of boosting. 
	
	\begin{enumerate}
		\item By Lemma~\ref{lem:failure-rate-of-rejection-sampler}, 
		$
		O( \frac{m_\rwkl}{\eps} \cdot \log(T/\rho) )
%		= O( \frac{m_\rwkl}{\eps} \cdot \log(6T/\rho) )
		$ examples suffice to guarantee $\rejectionsampler$ outputs $\bot$ with probability at most $\rho/(6T)$. Union bounding over two runs, this is at most $\rho/(3T)$. 
		
		\item By Lemma~\ref{lem:rejection-sample-then-run-reproducible-algo}, 
		running $\rwkl$ with reproducibility parameter $\rho/(6T)$ will add a $\rho/(3T)$ contribution to the non-reproducibility.
		
		\item We run $\rstat$ with reproducibility parameter $\rho/(3T)$. 
	\end{enumerate}
	
	\textbf{Sample Complexity: }
	There are two contributions to the sample complexity: samples used for the weak learner $\rwkl$, and samples used by $\rstat$ to estimate the density of measure $\mu_t$. Fresh samples are used for each of $T$ rounds of boosting. 
	Together, by Theorem~\ref{thm:reproducible-sq-oracle} and 
	the definition of $\vec s_1$ (in Algorithm~\ref{alg:reproducible-boosting}), the sample complexity is 
	$$
	O \left(
	T \cdot 
	\left(
	\frac{	m_{\rwkl(\rho/(6T))}}{\eps}
	%\footnote{Here, $m_\rwkl$ is a function and $\rho/(6T)$ is an input, not a multiplicative factor}
	\cdot \log(T/\rho)
	+ 
	\frac{\log (T/\rho)}{(\eps^2)(\rho)^2}	
	\right)
	\right)
	%%%%%%%%%%%%%%%%%%%%
	%%%%%%%%%%%%%%%%%%%%
	=
	\widetilde O 
	\left(
	\frac{	m_{\rwkl(\rho/(6T))}}{\eps^2 \gamma^2}
	%\footnote{Here, $m_\rwkl$ is a function and $\rho/(6T)$ is an input, not a multiplicative factor}
	+ 
	\frac{1}{\rho^2 \eps^3 \gamma^2}	
	\right)
	$$
	where the $\widetilde{O}$ notation hides $\log(1/(\rho \eps \gamma^2))$ factors
	and $m_{\rwkl(\rho/(6T))}$ denotes the sample complexity of $\rwkl$ with reproducibility parameter $\rho \eps \gamma^2$.

	\textbf{Failure Probability: }
	Assuming the weak learner returns correct hypotheses when it is reproducible, the boosting algorithm $\rboost$ is correct when it is reproducible, so the failure probability is bounded above by $\rho$.\footnote{
		A more precise sample complexity statement in terms of the failure probability $\delta$ can be obtained by unboxing the error probabilities. 
	The algorithm can fail if $\rejectionsampler$ outputs $\bot$,  if $\rwkl$ fails, and if $\rstat$ fails. 
	Bounding each of these quantities by $\delta/(3T)$ ensures that the $\rboost$ has failure rate $\delta$.} 	
	%One could unpack the failure probabiilies and get a more precise sample complexity statment in terms of deltas (failure probabilities of subroutines)
\end{proof}

\subsection{Rejection Sampling Lemmas}

Next, we show that reproducibility composes well with rejection sampling throughout the execution of $\rboost$.

\begin{lemma}[Failure Rate of $\rejectionsampler$]
	\label{lem:failure-rate-of-rejection-sampler}
	Let measure $\mu$ have density $d(\mu) \ge \eps/3$. 
	Let $\sall$ be a sample drawn i.i.d. from distribution $D$.
	If $|\sall| \ge \frac{24 \mtarget}{ \eps} \cdot \log (1/\delta)$,
	then  $\rejectionsampler(\sall, \mtarget, \mu; r)$ outputs $\bot$ with probability at most $\delta$. 
\end{lemma}

\begin{proof}
	The probability $\rejectionsampler$ outputs $\bot$ is precisely the probability a binomial random variable $X \sim B(|\sall|, d(\mu))$ is at most $\mtarget$. 
	By a Chernoff bound,
$	\Pr [X \le (1-.5) |\sall| \cdot d(\mu)]
	\le \exp(- |\sall| \cdot d(\mu)/8)
	\le \exp(- |\sall| \cdot \eps/24)$. 
	Thus, 	
	$\Pr[X \le \mtarget] 
	\le \delta	$.
\end{proof}

\begin{remark}
	The following is a justification of why we may assume $d(\mu) \ge \eps/3$ in the previous lemma.  
	
	When $\rejectionsampler$ is first called in round $1$ of $\rboost$, $\mu(x) = 1$ for all $x$, so $d(\mu) = 1$.
	In subsequent rounds $t \ge 2$, $\rejectionsampler$ is only called if, in previous round $t-1$, $\rstat$ estimated $d(\mu)$ to be at least $2\eps/3$. $\rstat$ is run with tolerance $\eps/3$, so $d(\mu) \ge \eps/3$ whenever $\rstat$ succeeds. Whenever we apply the above lemma, we are assuming the success of previous subroutines (by keeping track of and union bounding over their error). 

\end{remark}

The following lemma shows that rejection sampling before running a reproducible algorithm only increases the non-reproducibility $\rho$ by a factor of $2$. To be precise, we let $p$ denote the probability the rejection sampler returns $\bot$. However, when we apply this Lemma in the proof of Theorem~\ref{thm:reproducible-boosting}, we will have already accounted for this probability. 

\begin{lemma}[Composing Reproducible Algorithms with Rejection Sampling]
	\label{lem:rejection-sample-then-run-reproducible-algo}
	Let $\AAA(\vec{s}, r)$ be a $\rho$-reproducible algorithm with sample complexity $m$ 
	Let $\mu: \data \rightarrow [0,1]$.
	Consider $\BBB$, the algorithm defined by composing $\rejectionsampler(\vec{s}', m, \mu; r')$ with $\AAA(\vec{s}; r)$. 
	Let $q$ be the probability that $\rejectionsampler$ returns $\bot$. 
	Then $\BBB$ is a $2q + 2\rho$-reproducible algorithm.
\end{lemma}

\begin{proof}
	Since $\AAA$ is $\rho$-reproducible, 
	        $\Pr_{{\vec s_1},{\vec s_2},r} \left[
	\mathcal{A}({\vec s_1}; r) = \mathcal{A}({\vec s_2}; r) 
	\right]
	\ge 1 - \rho$.
	However, the rejection sampling is done with correlated randomness, so $\vec{s}_1$ and $\vec{s}_2$ are not independent. 
	Consider an imaginary third run of algorithm $\AAA(\vec s_3; r)$, where $\vec s_3$ is drawn from $D_\mu$ using separate randomness. 
	We will use a triangle-inequality-style argument (and a union bound) to derive the conclusion. 
	Conditioned on $\rejectionsampler$ not returning $\bot$, algorithm $\BBB(\vec s_1'; r' || r)$ %\footnote{Notation: $a || b$ denotes string $a$ concatenated with string $b$}
	  returns the same result as $\AAA(\vec s_3, r)$ (when both algorithms use randomness $r$ for the execution of $\AAA$) with probability at least $1-\rho$. The same statement holds for the second run $\BBB(\vec s_1'; r' || r)$. 
	Thus,
	 $$\Pr_{{\vec s'_1},{\vec s'_2}, r' || r} \left[
	\BBB({\vec s'_1}; r'||r) = \BBB({\vec s'_2}; r'|| r) | \text{ neither run outputs } \bot 
	\right]
	\ge 1 - 2\rho.$$
	
	Finally, $\BBB$ may fail to be reproducible if either $\rejectionsampler$ call returns $\bot$, so we union bound over this additional $2q$ probability. 
\end{proof}

\subsection{Reproducible Strong Halfspace Learner}

We give two reproducible strong learners for halfspaces by combining boosting algorithm $\rboost$ with  reproducible weak halfspace learners $\rhalfspacewkl$ and $\rhalfspacewklboxes$.

%say something here about runtime

\begin{corollary}
	\label{cor:reproducible-strong-halfspace-learner}
	Let $D$ be a distribution over $\R^d$, and let $f: \R^d \rightarrow \pmone$ be a halfspace with margin $\tau$ in $D$. 
	Let $\eps > 0$. 
	Then 
	\begin{itemize}
		\item 	Algorithm $\rboost$ run with weak learner $\rhalfspacewkl$
		$\rho$-reproducibly returns a hypothesis
		$\boldh$ such that,
		with probability at least $1-\rho$,
		$\Pr_{\vec{x} \sim D} [\boldh(\vec{x}) = f(\vec{x})]  \ge 1 - \eps$, 
		using a sample of size $
			\widetilde O 
		\left(
		\frac{d^{10/9}}{\tau^{76/9} \rho^{20/9} \eps^{28/9}}
		\right)
		$. 
		
		\item Algorithm $\rboost$ run with weak learner $\rhalfspacewklboxes$
		$\rho$-reproducibly returns a hypothesis
		$\boldh$ such that,
		with probability at least $1-\rho$,
		$\Pr_{\vec{x} \sim D} [\boldh(\vec{x}) = f(\vec{x})]  \ge 1 - \eps$, 
		using a sample of size $
		\widetilde O 
		\left(
		\frac{d^{15/4}}{\tau^{10} \rho^{5/2} \eps^{9/2}} 
		\right)
		$. 
	\end{itemize}

\end{corollary}

\begin{proof}
	For the first strong learner, we compose Theorem~\ref{thm:reproducible-boosting} with Theorem~\ref{thm:reproducible-halfspace-wkl}. 
		$\rhalfspacewkl$ has advantage $\gamma = \tau/4$, so
	$\rboost$ has round complexity $T = O(1/(\eps \gamma^2)) = O(1/(\eps \tau^2))$. 
	$\rboost$ runs $\rhalfspacewkl$ with parameter $\rho_\rwkl = \rho/6T$, so the sample complexity $m_\rhalfspacewkl$ is
$	O \left( \frac{d^{10/9}}{\tau^{58/9} \rho^{20/9} \eps^{10/9}} \right)$. Thus, the sample complexity for the boosting algorithm is
	
	$$
	\widetilde O 
	\left(
	\frac{d^{10/9}}{\tau^{76/9} \rho^{20/9} \eps^{28/9}}
	\right)
	$$
	
	For the second strong learner, we compose Theorem~\ref{thm:reproducible-boosting} with Theorem~\ref{thm:reproducible-halfspace-wkl-boxes}. 
	As before, $\rboost$ has round complexity $T = O(1/(\eps \tau^2))$ and runs $\rhalfspacewklboxes$ with parameter $\rho_\rwkl = \rho/6T$. 
	The sample complexity  $m_{\rhalfspacewklboxes}$ is
	$O\left (
	\left(\frac{d^{3/2}}{\tau^2 \rho_{\rwkl}} \right)^{5/2}
	\right) 
	= \left (
	\left(\frac{d^{3/2}}{\tau^4 \rho \eps} \right)^{5/2}
	\right) 
	$. 
	Thus, the sample complexity for the boosting algorithm is
	$
	\widetilde O 
	\left(
	\frac{	m_{\rwkl(\rho/(6T))}}{\eps^2 \gamma^2}
	%\footnote{Here, $m_\rwkl$ is a function and $\rho/(6T)$ is an input, not a multiplicative factor}
	+ 
	\frac{1}{\rho^2 \eps^3 \gamma^2}	
	\right)
	=
	\widetilde O 
	\left(
	\frac{	1}{\eps^2 \gamma^2}
	\left(\frac{d^{3/2}}{\tau^2 \rho_{\rwkl}} \right)^{5/2}
	\right)
	=
	\widetilde O 
	\left(
	\frac{d^{15/4}}{\tau^{10} \rho^{5/2} \eps^{9/2}} 
	\right)
	$.
\end{proof}

\begin{remark}
	$\rhalfspacewklboxes$ can be run in time polynomial in the input parameters, so the strong learner obtained by running $\rboost$ with weak learner $\rhalfspacewklboxes$ is a $\poly(1/\eps, 1/\rho, 1/\tau, d)$-time algorithm.
	However, the other weak learner $\rhalfspacewkl$ uses a foams construction subroutine from \cite{KDRW12} that takes expected exponential in $d$ runtime. The corresponding strong learner runs in time polynomial in $1/\eps, 1/\rho,$ and $1/\tau$, but exponential in $d$. 
\end{remark}

\subsection{Discussion}

%\subsubsection{Comparison to related boosting algorithms}

Algorithm~\ref{alg:reproducible-boosting} follows the smooth boosting framework of Servedio   \cite{Servedio:03jmlr}, which also shows how to boost a weak halfspace learner under a margin assumption on the data. They show that their boosted halfspace learner obtains a hypothesis with good margin on the training data, and then apply a fat-shattering dimension argument to show generalization to the underlying distribution with sample complexity $\tilde{O}(1/(\tau \epsilon)^2)$. Notably, this gives sample complexity independent of $d$. Moreover, their smooth boosting algorithm is tolerant to malicious noise perturbing an $\eta \in O(\tau \epsilon)$ fraction of its sample. 

A generic framework for differentially private boosting was given in \cite{BCS20}, with an application to boosting halfspaces. Their boosting algorithm also follows the smooth boosting framework, but uses a variant of the round-optimal boosting algorithm given in \cite{BHK09}. Their halfspace learner similarly requires a margin assumption on the data and tolerates random classification noise at a rate $\eta \in O(\tau\epsilon )$. They give two generalization arguments for their halfspace learner, both of which are dimension-independent. The first follows from prior work showing that differential privacy implies generalization \cite{BNSSSU16} and gives sample complexity $\tilde{O}(\frac{1}{\epsilon \alpha \tau^2} + \frac{1}{\epsilon^2\tau^2} + \alpha^{-2} + \epsilon^{-2})$ for approximate differential privacy parameters $(\alpha, \beta)$. The second follows from a fat-shattering dimension argument and gives a tighter bound of $\widetilde{O} \left(\frac{1}{\epsilon \alpha \tau^2} \right)$.  

Boosting algorithms have been thoroughly studied over the past few decades, and there are many types of boosting algorithms (e.g. distribution-reweighting, branching-program, gradient boosting) with different properties (e.g. noise-tolerance, parallelizability, smoothness, batch vs. filtering). It would be interesting to see which of these techniques can be made reproducible, and at what cost.

\section{SQ--Reproducibility Lower Bound}
\label{sec:coin-problem-sq-lb}

How much does it cost to make a nonreproducible algorithm into a reproducible one? In this section, we show a lower bound for reproducible statistical queries via a reduction from the coin problem. 

\begin{theorem}[SQ--Reproducibility Lower Bound]
	\label{thm:sq-reproducibility-lb}
	Let $\tau > 0$ and let $\delta \le 1/16$. 
	Let query $\phi: \data \rightarrow [0,1]$ be a statistical query. 
	Let $\AAA$ be a $\rho$-reproducible SQ algorithm for $\phi$ with tolerance less than $\tau$ and success probability at least $1 - \delta$. 
	Then $\AAA$ has sample complexity at least $m \in \Omega(1/(\tau^2 \rho^2))$.
\end{theorem}
%This theorem directly implies $\rho \in \Omega(1/\tau\sqrt{m})$. 
Note that this nearly matches the reproducible statistical query upper bound in Theorem~\ref{thm:reproducible-sq-oracle}, in the case that $\delta \in \Theta(\rho)$. 

Recall the coin problem: promised that a $0$-$1$ coin has bias either $1/2-\tau$ or $1/2 + \tau$ for some fixed $\tau > 0$, how many flips are required to identify the coin's bias with high probability?

\begin{proof}[Proof of Theorem~\ref{thm:sq-reproducibility-lb}]
A $\tau$-tolerant $\rho$-reproducible SQ algorithm $\AAA$ for $\phi$ naturally induces a $\rho$-reproducible algorithm $\BBB$ for the $\tau$-coin problem --- $\BBB$ runs $\AAA$ (the results of the coin flips are the $\phi(x)$'s), and $\BBB$ accepts (outputs $1$) if $\AAA$'s output is $\ge 1/2$, otherwise rejects. 
The success probability of $\BBB$ is at least that of $\AAA$. 
As $\AAA$ is reproducible for all distributions, $\BBB$ also satisfies the assumption in Lemma~\ref{lem:coin-problem-lb} that $\BBB$ is $\rho$-reproducible for coins with bias in $(1/2 - \tau, 1/2 + \tau)$. 
By Lemma~\ref{lem:coin-problem-lb}, any reproducible algorithm solving the coin problem with these parameter has sample complexity $m \in \Omega(1/(\tau^2 \rho^2))$, implying the lower bound. 
\end{proof}

\begin{lemma}[Sample Lower Bound for the Coin Problem]
	\label{lem:coin-problem-lb}
	Let $\tau < 1/4$ and $\rho < 1/16$. 
	Let $\BBB$ be a $\rho$-reproducible algorithm that decides the coin problem 
	with success probability at least $1 - \delta$ for $\delta = 1/16$.
	Furthermore, assume $\BBB$ is $\rho$-reproducible, even if its samples are drawn from a coin $\mathbf{C}$ with bias in $(1/2 - \tau, 1/2 + \tau)$. 
	Then $\BBB$ requires sample complexity $m \in \Omega(1/(\tau^2 \rho^2))$, i.e. $\rho \in \Omega(1/\tau\sqrt{m})$. 
\end{lemma}

\begin{proof}

Assume we have an algorithm $\BBB(b_1..b_m; r)$  of sample complexity $m$ so that (i) if the $b_i$'s are chosen i.i.d. in $\{0,1\}$ with bias $ 1/2-\tau$, $\BBB$ accepts with at most $\delta$ 
probability (over both random $r$ and the $b_i$'s), and (ii) if the $b_i$'s are drawn i.i.d. with bias $ 1/2+\tau$, $\BBB$ accepts with at least $1 -\delta$ probability.

%and for every $1/2 - \tau \le p \le 1/2+\tau$, 
%are randomly and independently sampled with expectation $p$.

Let $p \in [0,1]$ denote the bias of a coin. 
Since $\BBB$ is $\rho$-reproducible, $\BBB$ is $\rho$-reproducible for any distribution on $p$. 
In particular, pick $p \in_U [1/2-\tau, 1/2+\tau]$. 
Let $\mathbf C_{-\tau}$ denote a coin with bias $1/2 - \tau$, and let $\mathbf C_{+\tau}$ denote a coin with bias $1/2 + \tau$. 
By Markov's inequality, each of the following is true with probability at least $1 - 1/4$ over choice of $r$:
\begin{itemize}
	\item $\Pr_{b_1..b_m \sim_{\text{i.i.d.}} \mathbf C_{-\tau}} [\BBB(b_1..b_m ; r) \text{ accepts}] \le 4\delta$ 
	
	\item $\Pr_{b_1..b_m \sim_{\text{i.i.d.}} \mathbf C_{+\tau}} [\BBB(b_1..b_m ; r) \text{ accepts}] \ge 1 - 4\delta
	$
	
	\item When $p$ is chosen between $1/2-\tau $ and $1/2 + \tau$ uniformly, and then $b_1..b_m, b'_1..b'_m$
	are sampled i.i.d. with expectation $p$, 
	$\Pr[ \BBB(b_1..b_m;r)= A(b'_1..b'_m;r)] \ge 1 - 4\rho$.
	
\end{itemize}

By a union bound, there exists an $r^*$ so that every above statement is true.
Note that for any $p$, given $\sum b_i =j$, the samples $b_1..b_m$ are uniformly distributed among all Boolean vectors of Hamming weight $j$.
Let $a_j \eqdef \Pr[\BBB(b_1..b_m; r^*)$ accepts $ | \sum b_i = j ]$. 
Then the probability $\BBB$ accepts using $r^*$ on bits 
with bias $p$ is $\Acc(p)= \sum_j a_j {m \choose j}  p^j (1-p)^{m-j}$. 
In particular, this is a continuous and differentiable function.

Since $\Acc(1/2-\tau) < 4 \delta < 1/4$ and $\Acc(1/2+\tau) > 1-4 \delta> 3/4$, there is a $q \in (1/2-\tau, 1/2+ \tau)$ with $\Acc(q)=1/2$.  We show that $\Acc(p)$ is close to $1/2$ 
for all $p$ close to $q$ by bounding the derivative $\Acc'(p)$ within 
the interval $[1/4, 3/4]$, which contains $[1/2-\tau , 1/2+\tau ]$.

By the standard calculus formulas for derivatives,
\begin{equation*}
	\begin{split}
		\Acc'(p) 
		&= \sum_j a_j {m \choose j} (j p^{j-1}(1-p)^{m-j} - (m-j) p^{j}(1-p)^{m-j-1} )
		\\&= \sum_j a_j {m \choose j} p^j (1-p)^{m-j } ( j/p - (m-j)/(1-p)) 
		\\&= \sum_j a_j {m \choose j}  p^j(1-p)^{m-j}
		 (j-mp)/(p(1-p)).
	\end{split}
\end{equation*}
Since $1/4 < p < 3/4$, $p(1-p) > 3/16> 1/6$, and $0 \le a_j \le 1$.
So this sum is at most
$$ \sum_j {m \choose j} p^j(1-p)^{m-j} 6 |j-mp| = 6 \Exp_j [|j-mp|]$$ where the last expectation is over $j$ chosen as the sum of $m$ random Boolean variables of expectation $p$.
 This expectation is $O(m^{1/2})$ because the expectation of the absolute
value of the difference between any variable and its expectation is at most
the standard deviation for the variable.

 Since the derivative is at most $O(\sqrt{m})$, there is an interval $I$ of length $\Omega(1/\sqrt{m})$ around $q$ so that $1/3< \Acc(p) < 2/3$ for all $p$ in this interval.
Since $\Acc(p) \not \in (1/3, 2/3)$ at $p=1/2-\tau$ and $p =1/2+\tau$, interval $I$ is entirely contained in $(1/2-\tau, 1/2+\tau)$.
So, there is an  
$\Omega (1/ \tau \sqrt{m})$ chance that a random 
$p \in_U [1/2-\tau, 1/2+\tau]$ falls in interval $I$. 
For $p \in I$, there is a $2 \Acc(p) (1-\Acc(p)) > 4/9 $ conditional probability of non-reproducibility for $\BBB$.
Therefore, $\rho \ge \Omega(1/\tau \sqrt{m})$ and  
$m  \in \Omega (1/\tau^2 \rho^2)$. 
%matching the upper bound for statistical queries.

\end{proof}

\paragraph{Acknowledgements.}

The authors would like to thank Mark Bun, Cynthia Dwork, Max Hopkins, Pasin Manurangsi, Rahul Santhanam, Geelon So, Ryan Williams, and TPDP 2021 participants for interesting discussions.

\newpage

\bibliography{allrefs,toni-refs}

\newpage
\appendix

\section{Reproducibility: Alternative Definitions and Properties}
\label{app:additional-properties-of-reproducibility}

In this section, we consider a few alternative criteria for
reproducibility and show how they relate to our definition of
reproducibility. 
We also demonstrate other robustness properties of reproducibility
such as amplifying the parameters, as well as data/randomness reuse. 

\medskip

\noindent {\bf Alternative Definitions and Amplification.}
In the main body of the paper,
 we have chosen 
% (for simplicity) 
 to define
$\AAA$ as have two sources of random inputs:
samples ${\vec s}$ drawn from distribution $D$ and internal randomness $r$. $\AAA$ has no additional inputs.
However, we could more generally define $\AAA$ to have
additional, nonrandom inputs.
In this more general definition, we define $\AAA(x;\vec{s};r)$
where $\vec{s}$ and $r$ are as defined previously, and $x$ is an auxiliary
input (or tuple of inputs). $\AAA(x;\vec{s};r)$ is $\rho$-reproducible
with respect to distribution $D$ if for every input $x$,
$\AAA(x;\vec{s};r)$ is $\rho$-reproducible.
This definition generalizes both pseudodeterministic
algorithms (in which there is no underlying distribution, so $\vec{s}$ is
empty) as well as our definition of reproducible learning algorithms (in which there are no additional inputs, so $x$ is empty).

%\medskip
~

Rather than parameterize reproducibility by a single parameter $\rho$, one could use two variables $(\eta, \nu)$. 
% as in the following alternative definition where we
%view reprocibility as an event on the random strings used, with some random strings being highly reproducible and others not.

\begin{definition}[$(\eta, \nu)$-reproducibility]
	\label{def:two-parameter-reproducibility}
Let $\AAA(x;\vec{s}; r)$ be an algorithm, where $\vec{s}$ 
are samples from $D$, and $r$ is the internal randomness.
We say that a particular random string $r$ is  $\eta$-good for
$\AAA$ on $x$ with respect to $D$ if
there is a single ``canonical" output $Z_r$ such that
$\Pr[\AAA(x; \vec{s}; r)=Z_r] \geq 1-\eta$.
Then
$\AAA$ is $(\eta, \nu)$-reproducible with
respect to $D$ if, for each $x$,
the probability that a random $r$ is $\eta$-good for $\AAA$ (on $x$ and $D$)
is at least $1-\nu$.
\end{definition}

$(\eta, \nu)$-reproducibility is qualitatively the same as $\rho$-reproducibility, but might differ by polynomial factors.
%polynomially quantitatively.
If $\AAA$ is $(\eta, \nu)$-reproducible,
then $\AAA$ is $\rho$-reproducible on $D$, where $\rho \le 2\eta + \nu$. The probability that two runs of $\AAA$, using the same internal randomness $r$, output different results
is at most $\Pr[r \text{ not }\eta\text{-good}]$ plus the probability that at least one run is not the special output $Z_r$ (conditioned on $r$ being $\eta$-good).
In the other direction,
if $\AAA$ is $\rho$-reproducible,
then $\AAA$ is $(\rho/\nu, \nu)$-reproducible
for any $\rho \le \nu < 1$.
Say there is a $\nu$ probability that $r$ is not $\eta$-good. Conditioned on picking a not $\eta$-good $r$,
there is a conditional (at least) $\eta$ probability of the second run of $\AAA$ returning something different than the first 
run.\footnote{If $1 - \eta \ge 1/2$, then the probability that two runs of $\AAA$ using the same randomness returns the same result is at most $(1-\eta)^2 + \eta^2$, i.e., when $\AAA$ has only two possible outputs. This is less than $1-\eta$, assuming $1-\eta \ge 1/2$. If $1-\eta < 1/2$, then there must be more than two outputs, and the probability of nonreproducibility is again larger than $\eta$.}
Thus, $\rho \ge \eta \nu$.

A similar definition called ``pseudo-global stability", developed independently to our work, appears in \cite{GKM:21}. That definintion parametrizes by the sample complexity $m$ and does not explicitly parametrize by auxiliary input $x$. Additionally, their definition includes an $(\alpha, \beta)$-accuracy guarantee on $Z_r$, the very likely output. To keep both definitions consistent with their original conventions, we write $\eta' \eqdef 1 - \eta$ and $\nu' \eqdef 1- \nu$. 
%If algorithm $\AAA$ is pseudo-globablly stable with $(\alpha, \beta)$-accuracy,...

\begin{definition}[Pseudo-global stability, Definition 15 in \cite{GKM:21}]
	\label{def:pseudo-global-stability}
	A learning algorithm $\AAA$ with sample complexity $m$ is said to be $(\alpha, \beta)$-accurate, $(\eta', \nu')$-pseudo-globally stable if there exists a hypothesis $h_r$ for every $r \in \supp(R)$ (depending on $D$) such
	that 
	$\Pr_{r \sim R} [\err_D(h_r) \le \alpha] \ge 1-\beta$ and
	$$\Pr_{r \sim R} \left[ 
	\Pr_{\vec{s} \sim D^m}[\AAA(\vec{s}; r) = h_r] 
	\ge \eta'
	\right] \ge \nu'$$
where $\vec{s}$ is a sample of $m$ (labeled) examples $(x_i, y_i)$ drawn from distribution $D$. 
\end{definition}

The final condition of Definition~\ref{def:pseudo-global-stability} is equivalent to saying that i) a randomly chosen $r$ is $\eta'$-good with probability at least $\nu'$, and ii) for every $r$, $h_r$ is the output that witnesses the $\eta$-goodness. 
Carrying the accuracy guarantee through the previous argument, an $(\alpha, \beta)$-accurate $(\eta', \nu')$-pseudo-globally-stable algorithm $\AAA$ implies a $(2(1-\eta') + (1-\nu')) = (2\eta + \nu)$-reproducible algorithm $\AAA$ also with $(\alpha, \beta)$-accuracy.

If we are willing to increase the sample complexity of $\AAA$, we can make the connection stronger:

\begin{theorem}[Amplification of Reproducibility]
	\label{thm:amplification-of-reproducibility}
	Let $0 < \eta, \nu, \beta < 1/2$ and $m > 0$.
	Let $\AAA$ be an $(\eta, \nu)$-reproducible algorithm for distribution $D$ with sample complexity $m$ and failure rate $\beta$. 
	If $\rho > 0$ and $\nu + \rho < 3/4$,
	then 
	there is a $\rho$-reproducible algorithm $\mathcal{A'}$ for $D$ with
	sample complexity $m' = \widetilde{O}( m (\log 1/\beta)^3/(\rho^2(1/2-\eta)^2)$ and failure rate
	%error probability
	at most $O(\beta + \rho)$.
	The construction of $\mathcal{A'}$ does not depend on $D$.
\end{theorem}

\begin{proof}
Set $k = 3 \log 1/\beta$.
For each random string $r$, let  $D_r$ be the  distribution on outputs of
$\AAA(x; \vec{s};r)$ (over random $\vec{s}$).
Algorithm $\mathcal{A'}$ randomly picks $k$-many strings $r_1, \dots, r_k$, runs the reproducible heavy-hitters algorithm (Algorithm~\ref{alg:heavy-hitter-rep}) on the distributions $D_{r_1}, \dots, D_{r_k}$, and outputs the first returned heavy-hitter (or $\perp$ if each subroutine returns the empty list). 
We say there are $k$ \textit{rounds} of $\mathcal{A'}$, one per random string $r$. 

The reproducibility of $\rhh$ implies the reproducibility of $\mathcal{A'}$. We show that a heavy-hitter in $D_r$ for randomly chosen $r$ is often a correct answer, except with probability comparable to $\beta$.

By definition, $r$ is $\eta$-good iff
$D_r$ has a $1-\eta$ heavy-hitter.
Since $\eta < 1/2$, this heavy-hitter will be unique, and there
will be no other $1-\eta > 1/2$ heavy-hitters.
Given any $r$, we can draw from distribution $D_r$ by running algorithm $\AAA$ with fresh samples $\vec{s}$.
Consider running the reproducible heavy-hitters algorithm 
 with parameters 
 $v = (3/2-\eta)/2$,
$\epsilon = (1/2 - \eta)/2$, 
and reproducibility $\rho' = \rho/k$.
These are chosen so that $v + \eps = 1-\eta$ and $v - \eps = 1/2$. 
If $r$ is $\eta$-good, then $\rhh$ will return the (unique) majority element for $D_r$ with probability at least $1-\rho/k$. 
If $r$ is not $1/2$-good\footnote{Since $\eta < 1/2$ by assumption, $r$ being not $1/2$-good implies $r$ is not $\eta$-good.}, 
the reproducible heavy-hitters algorithm 
with the same parameters 
will return the empty list with probability at least $1- \rho/k$.
%If $r$ is not $1/2$-good, the probability that the reproducible heavy-hitters algorithm outputs a non-empty list is at most $\rho/k$.

Next, we compute the conditional probability that the first element returned by $\rhh$ is correct.
The probability that $\rhh$ produces an empty list in one round is at most $\nu$ (when the randomly chosen $r$ is not $\eta$-good) plus $\rho/k$ (when $r$ is $\eta$-good but the heavy-hitters algorithm fails). 
At most a $(2\beta)$-fraction of random strings $r$ 
satisfy both of the following two conditions:
i) $D_r$ has a majority element $Z_r$ and 
ii) $Z_r$ is an incorrect output.
Thus, the conditional probability of outputting an incorrect answer, given $\rhh$ produces a non-empty output, is at most $(2 \beta  + \rho/k)/(1-\nu-\rho/k)$. By assumption, $\nu + \rho/k < 3/4$, so this is $O(\beta + \rho)$. 

So far, we have bounded the probability that $\mathcal{A'}$ returns an incorrect answer. $\mathcal{A'}$ could also fail if $\rhh$ returns the empty list in each of $k$ rounds. Since $\nu + \rho/k < 3/4$, this happens with probability at most $(3/4)^{k} \le \beta$.
So, the overall probability of error is at most $O(\beta+\rho)$.

If two runs of $\mathcal{A'}$ use the same $r_i$'s and same randomness for each heavy-hitters call,
they only produce different answers if a pair of $\rhh$ calls produces
different answers in the same round.  By the reproducibility of $\rhh$, this only
happens with probability $\rho/k$ each round, 
for a total non-reproducibility probability at most $\rho$.

$\mathcal{A'}$ calls $\rhh$ $k = O(\log 1/\beta)$ times. Each example used by $\rhh$ is created by running $\AAA$, which has sample complexity $m$. 
By Lemma~\ref{lem:rHeavyHitters-correctness}, $\rhhfullrhoprime$ has sample complexity
$
\widetilde{O}\left(
\frac{1}
{\rho'^2\epsilon^2(v - \epsilon)^2}
\right)$. 
Substituting in $\rho' = \rho/k, \eps = (1/2-\eta)/2,$ and $v-\eps = 1/2$,
$\AAA$ has sample complexity
$km \cdot \widetilde{O}\left(\frac{k^2}{\rho^2(1/2-\eta)^2} \right)
= 
\widetilde{O}\left(\frac{m \log^3 (1/\beta)}{\rho^2(1/2-\eta)^2} \right)
$.
%If $\eta$ is bounded away from $1/2$ by some constant, this can be simplified to $\widetilde{O}\left(m \log^3 (1/\beta)/\rho^2\right)$.
\end{proof}

\begin{corollary}
	Let $\alpha > 0$ and $\rho < 1/4-\alpha$. 
Let $\AAA$ be a $\rho$-reproducible algorithm
 using $m$ samples is correct except with error at most $\beta$.
Then for arbitrary $\rho'$ satisfying $\rho > \rho' >0$,
there is a $\rho'$-reproducible algorithm $\mathcal{A'}$ with sample complexity
$m' = \widetilde{O}\left(\frac{m \log^3 (1/\beta)}{\rho'^2 \alpha^2} \right)$
%$m'= \widetilde{O}(m (\log 1/\beta)^3/(\rho')^2)$ 
that is correct except with error at most $O(\beta + \rho')$.
\end{corollary}

\begin{proof}
	By the arguments immediately after Definition~\ref{def:two-parameter-reproducibility}, a $\rho$-reproducible algorithm implies a $(\rho/x, x)$-reproducible algorithm. Choosing $x = 1/2-\alpha$ allows us to apply Theorem~\ref{thm:amplification-of-reproducibility} for any $\rho < 1/4$. 
	The $(1/2-\eta)$ term in Theorem~\ref{thm:amplification-of-reproducibility} simplifies to $\alpha/(1-2\alpha)$ in this context. When $\alpha$ can be chosen as a constant, the sample complexity simplifies to
	$m' = \widetilde{O}(m \log^3 (1/\beta)/\rho'^2)$.
\end{proof}

%Proof: For some $\nu < 1/2 - \beta$ and $\eta < 1/2$, $A$ is $(\eta, \nu)$ reproducible.

%---------------

\medskip

\noindent{\bf Public versus Private Randomness.} We define reproducibility as the probability 
that when run twice using the same (public) randomness, but with independently chosen data samples, the algorithm returns the same answer.
In \cite{GrossmanLiu:19}, the authors define a related concept, but divide 
up the randomness into two parts, where only the first randomness part gets reused in the second run of the algorithm. 
In their applications, there are no data samples, so re-running the algorithm using identical randomness would always give identical results; rather, they were trying to minimize the amount of information about the random choices that would guarantee reproducibility, i.e., minimize the length of the first part.  

Similarly, we could define a model of reproducibility that involved two kinds
of random choices.  
Define  
$\AAA(x; \vec{s};r_{\textrm pub}, r_{\textrm priv})$, 
$\vec{s}=(s_1,\dots, s_m)$ to be $\rho$-reproducible with respect to 
$r_{\textrm pub}$ and $D$ if for every $x$,
random $\vec{s_1}$ and $\vec{s_2}$ drawn from $D^m$,
%=(s_1,..s_m)$, \vec{s'}=( s'_1,..s'_m)$  
%$s_i,s'_i$ independently chosen according to $D$, 
and random $r_{\textrm pub},r_{\textrm priv}, r'_{\textrm priv}$,
$$\Pr[\AAA(x; \vec{s_1}; r_{\textrm pub}, r_{\textrm priv})= \AAA(x; \vec{s_2}; r_{\textrm pub}, r'_{\textrm priv})]  \ge  1-\rho.$$

If we want to minimize the amount of information we need to store
to guarantee reproducibility, keeping $r_{\textrm priv}$ and $r_{\textrm pub}$ distinct
may be important.  
However, if all we want is to have a maximally reproducible algorithm, 
the following observation shows that it is always better to make the entire randomness public.

\begin{lemma}
If $\AAA(x; \vec{s}; r_{\textrm pub}, r_{\textrm priv})$
is $\rho$-reproducible with respect to $r_{\textrm pub}$ over $D$,
then $\AAA(x;\vec{s}; r_{\textrm pub}, r_{\textrm priv})$ is
 $\rho$-reproducible with respect to $(r_{\textrm pub}, r_{\textrm priv})$ over $D$.
\end{lemma}

\begin{proof}
We show for each value of $x$ and $r_{\textrm pub}$, 
$$
\Pr[\AAA(x; \vec{s_1}; r_{\textrm pub}, r_{\textrm priv})=  
\AAA(x; \vec{s_2}; r_{\textrm pub}, r'_{\textrm priv})]  \le  
\Pr[\AAA(x; \vec{s_1}; r_{\textrm pub}, r_{\textrm priv}) = 
 \AAA(x; \vec{s_2}; r_{\textrm pub}, r_{\textrm priv})].$$  
Fix $x$ and $r_{\textrm pub}$.  For each possible value $R$ of $r_{\textrm priv}$ and
each possible output $Z$, let 
$q_{R,Z} = \Pr[\AAA[(x;\vec{s_1};r_{\textrm pub}, R)] =Z$,
%for $s_i$ sampled independently from $D$, 
and let $\vec{q}_R$ 
be the vector indexed by $Z$ whose $Z^{th}$ coordinate is $q_{R,Z}$.  
Then 

$$\Pr[\AAA[(x; \vec{s_1}; r_{\textrm pub}, R)= 
\AAA(x; \vec{s_2}; r_{\textrm pub}, R)]   = 
\sum_Z (q_{R,Z})^2 = ||\vec{q}_R||^2_2,$$
and   
$$\Pr[\AAA(x; \vec{s_1}; r_{\textrm pub}, R) =
\AAA(x; \vec{s_2}; r_{\textrm pub}, R')]   = 
\sum_Z (q_{R,Z}q_{R',Z}) =  \langle \vec{q}_R,\vec{q}_{R'} \rangle.$$

Thus,
\begin{align*}
\Pr[\AAA(x; \vec{s_1}; r_{\textrm pub}, r_{\textrm priv}) 
=  
\AAA(x; \vec{s_2}; r_{\textrm pub}, r'_{\textrm priv})] 
& = \Exp_{R,R'} [\langle \vec{q}_R, \vec{q}_{R'} \rangle] \\
& \le \Exp_{R,R'} [||\vec{q}_R||_2 ||\vec{q}_{R'}||_2  ]  \\
& = (\Exp_R [ ||\vec{q}_R||_2])^2 \\
& \le \Exp_R [||\vec{q}_R||_2^2 ] \\
& = \Pr[\AAA(x; \vec{s_1}; r_{\textrm pub}, r_{\textrm priv})
 = \AAA(x; \vec{s_2}; r_{\textrm pub}, r_{\textrm priv})]. 
\end{align*}
\end{proof}

We will implicitly use this observation in the boosting algorithm section,
since it will be convenient to think of the two runs of the boosting algorithm as picking samples each step independently, when using the same random string
would create some correlation.

\medskip

\noindent\textbf{Reproducibility Implies Generalization.}
We show that a hypothesis generated by a reproducible algorithm has a high probability of having generalization error close to the empirical error. 
Let $h$ be a hypothesis, $c$ be a target concept, and $D$ be a distribution. 
The \textit{risk} (generalization error) of $R(h) \eqdef \Pr_{x \sim D}[h(x) \ne c(x)]$. 
If $\vec{s}$ is a sample drawn i.i.d. from $D$, then the \textit{empirical risk} $\widehat{R}_\vecs(h) \eqdef \Pr_{x \in \vecs}[h(x) \ne c(x)]$. 
%We assume these risks are bounded in the interval $[0,1]$.
%, although the following result can be generalized to risks in other bounded intervals by rescaling. 

\begin{lemma}[Reproducibility Implies Generalization]
	\label{lem:reproducibility-implies-generalization}
	Let sample $\vecs \sim D^n$, and let $\delta > 0$.
	Let $h$ be a hypothesis output by $\rho$-reproducible learning algorithm $\AAA(\vecs; r)$, where $r$ is a random string. 
	Then, with probability at least $1-\rho-\delta$ over the choice of $\vecs$ and $r$,  
	$R(h) \le \widehat{R}_\vecs(h) + \sqrt{\ln(1/\delta)/(2n)}$.
\end{lemma}

\begin{proof}
%	Let $\vecsone$ and $r$ denote the sample and randomness used by $\AAA$ to generate $h$. Let $m$ denote the sample complexity of $\AAA$. 

	Consider running $\AAA(\vecstwo; r)$, where $\vecstwo$ is an independent sample of size $m$ drawn from $D$, but $r$ is the same as before. Let  $h_2$ denote the returned hypothesis. Since $h_2$ is independent of $\vecs$, 
	$
	\Pr_{\vecs \sim D^n} 
	[ \widehat{R}_{\vecs}(h_2) - R(h_2) \ge \eps ] 
	\le \exp(-2n\eps^2)
	$
	for $\eps>0$ by Hoeffding's inequality.
	By the reproducibility of $\AAA$, $h_2 = h$ with probability at least $1-\rho$. 
	By a union bound,
	$R(h) \ge \widehat{R}_{\vecs}(h) + \sqrt{\ln(1/\delta)/2n}
	$
	with probability at least $1- \rho - \delta$.
\end{proof}

In the above argument, we use the definition of reproducibility to create independence between $\vecs$ and $h$, allowing us to use Hoeffding's inequality.

\iffalse
Note: the following paragraph of exposition no longer matches the proof of the previous lemma
 
By a similar argument, one can show a bound between the risk and the empirical risk of $h$ on $\vecsone$, the sample used to produce $h$. Specifically, $R(h) \le \widehat{R}_\vecsone(h) + \sqrt{\ln(1/\delta)/(2m)}$ with probability at least $1-\rho-\delta$. This follows from noticing that $\vecsone$ (like $\vecs$) is independent from $h_2$. 
\fi

%\subsection{Connections to Data Reuse}
%\label{sec:adaptivity}

\medskip

\noindent\textbf{Connections to Data Reuse.}
We consider the adaptive data analysis model discussed in \cite{Dworketal:2014} and \cite{Dworketal:2015}, and we prove that reproducible algorithms are resiliant against adaptive queries (Lemma~~\ref{lem:data-reuse}). The proof is via a hybrid argument.

% normal version of reproducibility
%\reproducibilityfull

\begin{lemma}[Reproducibility $\Longrightarrow$ Data Reusability]\label{lem:data-reuse}
	Let $D$ be a distribution over domain $\X$. Let $\mathcal{M}$ be a mechanism that answers queries of the form 
	$q: \X \rightarrow \{0,1\}$ by drawing a sample $S$ of $n$ i.i.d. examples from $D$ and returning answer $a$. 
	% and returning $a \approx \tfrac{1}{n}\sum_{x\in S}q(x)$. 
	Let $\mathcal{A}$ denote an algorithm making $m$ \emph{adaptive} queries, chosen from a set of queries $Q$, so that the choice of $q_i$ may depend on $q_j, a_j$ for all $j < i$. Denote by $[\mathcal{A}, \mathcal{M}]$ the distribution over transcripts $\{q_1, a_1, \dots q_m, a_m\}$ of queries and answers induced by $\mathcal{A}$ making queries of $\mathcal{M}$. 
	Let $\mathcal{M}'$ be a mechanism that behaves identically to $\mathcal{M}$, except it draws a single sample $S'$ of $n$ i.i.d. examples from $D$ and answers all queries with $S'$. 
	
	If $\M$ answers all queries $q \in Q$ with $\rho$-reproducible procedures, 
	then
	$SD_{\Delta}([\mathcal{A}, \mathcal{M}], [\mathcal{A}, \mathcal{M}']) \leq (m-1)\rho$, 
	where $SD_{\Delta}(D_1, D_2)$ denotes the statistical distance between distribtuions $D_1$ and $D_2$. 
	
\end{lemma}

\begin{proof}
	For $i \in 
	%\{1, \dots, m\}
	[m]$, let $[\AAA, \M_i]$ denote the distribution on transcripts output by algorithm $\AAA$'s interaction with $\M_i$, where $\M_i$ is the analogous mechanism that
	draws new samples $S_1, \dots, S_{i}$ for the first $i$ queries, and reuses sample $S_i$ for the remaining $m-i$ queries. 
	Note that $\M' = \M_1$ and $\M = \M_m$. 
	
	For $i \in [m-1]$, consider distributions 
	$[\AAA, \M_i]$ and $[\AAA, \M_{i+1}]$. 
	We will bound the statistical distance by a coupling argument. 
	Let $S_1, \dots, S_{i+1}$ denote random variables describing the samples used, and let $r$ denote the randomness used over the entire procedure. 
	$[\AAA, \M_i]$ can be described as running the entire procedure (with randomness $R$) on 
	%	$S_j$ for $j \in [i-1]$ and on $S_{i+1}$ for $j \in \{i, i+1, \dots, m\} $
	$S_1, \dots, S_{i-1}, S_{i+1}, S_{i+1}, \dots, S_{i+1}$, 
	and $[\AAA, \M_{i+1}]$ can be described as running the entire procedure (with randomness $R$) on 
	$S_1, \dots, S_{i-1}, S_{i}, S_{i+1}, S_{i+1}, \dots, S_{i+1}$.
	
	These distributions are identical for the first $i-1$ queries and answers, so the $i$'th query $q_i$ is identical, conditioned on using the same randomness.
	Both $S_i$ and $S_{i+1}$ are chosen by i.i.d. sampling from $D$, so 
	$\Pr_{S_{i}, S_{i+1}, r} \left[
	\mathcal{A}( q_i, S_{i+1} ; r) 
	= A( q_i, S_i ; r)
	\right]
	\ge 1- \rho
	$
	by reproducibility.
	Conditioned on both transcripts including the same $(i+1)$'th answer $a_{i+1}$ (and continuing to couple $S_{i+1}$ and $r$ for both runs), the remaining queries and answers $q_{i+1}, a_{i+1}, \dots, q_m, a_m$ is identical.
	Therefore, 
	$SD_{\Delta}([\AAA, \M_i], [\AAA, \M_{i+1}]) 
	\leq \rho$ for all $i \in [m-1]$. 
	Unraveling, $SD_{\Delta}([\mathcal{A}, \mathcal{M}], [\mathcal{A}, \mathcal{M}']) \leq (m-1)\rho.$
\end{proof}

\begin{remark}
	This connection may be helpful for showing that reproducibility cannot be achieved efficiently in contexts where data reuse is not efficiently achievable. 
\end{remark}

\section{Concentration of Sum of Vectors}
\label{apps:concentration-inequality}

In this Section, we use Azuma's inequality to prove a concentration bound on the sum of vectors from a distribution. 

Let $D$ be a distribution on $\R^n$.
Let ${\bf v} = \{ {\bf v_1}, \ldots, {\bf v_T}\} \in D^T$
be a random sample of $T$ vectors from ${D}$ with the
following properties:
\begin{enumerate}
	\item $\Exp_{{\bf v} \in {\ D}^T} [ \sum_{i=1}^{T} {\bf v_i}] - 
	\Exp_{v \in {\ D}} [v]  =0$.
	\item $\forall v \in {\ D}$, $|| v ||_2 \leq c$.
\end{enumerate}

The following lemma shows that the length of 
${\bf v^{1}} + {\bf v^{2}} + \ldots + {\bf v^{T}}$
is tightly concentrated.
%the probability that
%$||\sum_{i=0}^T v_i ||_2 \geq \sqrt{T} + \delta$.

\concentrationinequalitylem*

The intuition behind Lemma~\ref{lem:concentration-inequality} is similar
to the one-dimensional case, where ${\ D}$ is a distribution
over $(-1,1)$, 
${\bf v} \in {\ D}^T$, and $\sum_{i=1}^T {\bf v_i}$ is
concentrated around zero, with standard deviation ${\sqrt T}$.
Let ${\bf v^{\leq i}}$ denote $\sum_{i=1}^i {\bf v_i}$.
In the one-dimensional case, we can prove concentration
of ${\bf v^{\leq T}}$ via a Chernoff or martingale argument 
since the expected value of ${\bf v^{\leq i}}$ (the sum of the first $i$ numbers) is equal to ${\bf v^{\leq i-1}}$.
However for the higher dimensional case, ${\bf v^{\leq i}}$ is
now the sum of the first $i$ vectors, and it is in general not the
case that the expected length of ${\bf v^{\leq i}}$ is
equal or even not much larger than the length of ${\bf v^{\leq i-1}}$.
However, if the length of ${\bf v^{\leq i-1}}$ is sufficiently large (greater than ${\sqrt T}$), then
$\Exp[||{\bf v^{\leq i}}||_2 ~|~ {\bf v^{\leq i-1}}]$ can be upper bounded (approximately) by $||{\bf v^{\leq i-1}}||_2 + 1/{\sqrt T}$. 
Therefore, if we want to bound the probability that the length of
${\bf v^{\leq T}}$ is
large (at least ${\sqrt T} + \Delta$),
there must be some time $t$ such that the vector ${\bf v^{\leq t}}$ is outside of the ball of radius $\sqrt{T}$ around the origin, and never returns.
So we can bound the probability that $||{\bf v^{\leq T}}||_2 \geq {\sqrt t} + \Delta$, by considering the sequence of
random variables ${\bf x^{\leq t}},\ldots,{\bf x^{\leq T}}$ such that
${\bf x^{\leq t}}$ is equal to the length of ${\bf v^{\leq t}}$, and for each $t' \geq t$,
${\bf x^{\leq t'}}$ is the length of ${\bf v^{\leq t'}}$ minus
a correction term (so that we can upper bound 
$\Exp[{\bf x^{\leq t'+1}} ~|~ {\bf x^{\leq t'}}]$ by ${\bf x^{\leq t'}}$.)
We will show that ${\bf x^{\leq t}},\ldots,{\bf x^{\leq T}}$ is a supermartingale
where $|{\bf x^{\leq t' +1}} - {\bf x^{\leq t'}}|$ is bounded by a constant, and then the concentration inequality will follow from Azuma's Lemma.
%close to the length at time $i$), and that the increase
%in length from $i$ to $i+1$ is at most 1.

%vector $v_0 \in \R^n$ of length $\sqrt{T}$, and defining
%a sequence of random variables $X_1,\ldots,X_{T'}$, and $V_1,\ldots,V_d$
%where for $i \in [T]$, $V_i$ is a random vector from ${\ D}$
%and 
%$X_i$ is the length of $v_0 + \sum_{j=1}^i V_i$.
%We then define a modified sequence, $\overline{X_1},\ldots,\overline{X_T}$,
%which is obtained from $X_1,\ldots,X_T$ by stopping early at
%some time $i<T$ if $X_i$ drops below $\sqrt{T}-\frac{i}{3\sqrt{T}}$.
%The upper bound follows by Azuma's inequality
%since $\{\overline{X_i}}_i$ is a supermartingale.

\begin{definition}
	Let ${\ D}$ be a distribution over $\R^n$ satisfying
	the above two properties. 
	\begin{enumerate}
		\item Let ${\bf v} = \{{\bf v_1},\ldots,{\bf v_{T'}}\} \in {\ D}^{T'}$ be a 
		sequence of $T' \leq T$ random variables, and 
		let ${\bf v_0} \in \R^n$ have length $\sqrt{T}$.
		For $0 \leq i \leq T'$, let ${\bf v^{\leq i}} = \sum_{i=1}^{T'} {\bf v_i}$.
		
		\item The {\it stopping} {\it time} $\tau \in [T']$ (with respect to $\{{\bf v^{\leq i}}\}$)
		%$\{{\bf x^{\leq i}}\}_i$ is equal to:
		is equal to: 
		$$\text{min} \{ \{i \in [T'] ~|~ ||{\bf v^{\leq i}}||_2 < \sqrt{T} \} \cup \{ T'\}\}.$$
		That is, $\tau$ is the first time $i$ such that
		the length of ${\bf v^{\leq i}}$ drops below
		$\sqrt{T} + \frac{i}{3\sqrt{T}}$ (and otherwise
		$\tau =T'$).
		
		%\begin{definition}
		%Let ${\ D}$ be a distribution over $\R^n$ satisfying
		%the above two properties. Let $v_0 \in \R^n$ have length $\sqrt{T}$,
		%and let ${\bf v} = \{{\bf v_1},\ldots,{\bf v_{T'}}\} \in {\ D}^{T'}$ be a 
		%sequence of $T' \leq T$ random variables.
		%For $i \in [T']$, let ${\bf v^{\leq i}} = v_0 + \sum_{i=1}^{T'} {\bf v_i}$,
		%and let 
		%${\bf x^{\leq i}} = || {\bf v^{\leq i}}||_2 - \frac{i}{3\sqrt{T}}$.
		%\end{definition}
		
		\item For each $i \in [T']$, we define the sequence
		of random variables ${\bf x^{\leq 0}}, {\bf x^{\leq 1}},{\bf x^{\leq 2}},\ldots, {\bf x^{\leq T'}}$ where 
		${\bf x^{\leq 0}} = ||{\bf v^0}||_2 = {\sqrt T}$, and
		for all $i \geq 1$, ${\bf x^{\leq i}}$ will be the adjusted length of the first $i$ vectors, $||{\bf v^{\leq i}}||$ with stopping condition $\tau$:
		\begin{equation*}
		{\bf x^{\leq i}} =
		\begin{cases}
		||{\bf v^{\leq i}}||_2 - \frac{ci}{2\sqrt{T}} & \text{if}~ \tau > i \\
		{\bf x^{\leq \tau}} & \text{otherwise}
		\end{cases}
		\end{equation*}
	\end{enumerate}
\end{definition}

\begin{claim}
	The sequence of random variables ${\bf x^{\leq 1}},\ldots,{\bf x^{\leq T'}}$
	is a supermartingale.
\end{claim}

\begin{proof}
	We need to show that for every $i \in [T']$,
	$\Exp[{\bf x^{\leq i}} ~|~ {\bf x^{\leq i-1}}] \leq {\bf x^{\leq i-1}}$.
	Fix $i \in [T']$; if $\tau \leq i-1$ then 
	${\bf x^{\leq i}}={\bf x^{\leq i-1}}$ so the condition holds.
	Otherwise assume that $\tau \geq i$.
	%Since $\tau \geq i$, $||{\bf v^{\leq i-1}}||_2$ is at least $\sqrt{T}$.
	%${\bf x^{\leq i-1}} = || {\bf v^{\leq i-1}}||_2 - \frac{i-1}{3\sqrt{T}}$
	%is at least
	%$\sqrt{T} + \frac{i-1}{3\sqrt{T}}$ and therefore 
	%$||{\bf v^{\leq i-1}}||_2$ is at least $\sqrt{T}$.
	Since 
	$$\Exp[{\bf x^{\leq i}} ~|~ {\bf x^{\leq i-1}}]  =   \Exp[{\bf x^{\leq i}} ~|~ {\bf v^{\leq i-1}}]  = 
	\Exp[||{\bf v^{\leq i-1}} + {\bf v_i}||_2] - \frac{ci}{2\sqrt{T}}$$
	and 
	${\bf x^{\leq i-1}} = ||{\bf v^{\leq i-1}}||_2 - \frac{c(i-1)}{2\sqrt{T}},$
	it suffices to show that
	$ \Exp[||{\bf v^{\leq i-1}} + {\bf v_i}||_2 \leq ||{\bf v^{\leq i-1}}||_2 + \frac{c}{2\sqrt{T}}.$
	
	%That is, we want to show that the expected increase in the length 
	%of ${\bf v^{\leq i-1}}$ versus ${\bf v^{\leq i-1}} + {\bf v_i}$ is at most $\frac{1}{3\sqrt T}$.
	To prove this, we can write ${\bf v_i} = {\bf v_i^{\parallel}} + {\bf v_i^{\orth}}$ where
	${\bf v_i^{\parallel}}$ is the component of ${\bf v_i}$ in the direction of 
	${\bf v^{\leq i-1}}$,
	and ${\bf v_i^{\orth}}$ is the orthogonal component.
	Since the expected length of ${\bf v^{\leq i-1}}+{\bf v_i^{\parallel}}$ is
	equal to the length of ${\bf v^{\leq i-1}}$ (by property 1), we just have to show
	that the expected length of ${\bf v^{\leq i-1}} + {\bf v_i^{\orth}}$ is at most
	$\frac{c}{2\sqrt{T}}$.
	Since ${\bf v_i}$ has length at most $c$, so does ${\bf v_i^{\orth}}$,
	so we have:
	\begin{eqnarray*}
		\Exp[|| {\bf v^{\leq i-1}} + {\bf v_i^{\orth}}||_2]
		\leq  (||{\bf v^{\leq i-1}}||^2_2 + c)^{1/2} 
		\leq  \frac{c}{2\sqrt{T}}
	\end{eqnarray*}
	where the last inequality holds since $\tau \geq i$ implies
	$||{\bf v^{\leq i-1}}||_2 \geq \sqrt{T}$.
\end{proof}

\begin{claim}
	For all $i$, $|{\bf x^{\leq i}} - {\bf x^{\leq i-1}}| \leq c.$
\end{claim}

\begin{proof}
	Since ${\bf v_i}$ has length at most $c$
	the absolute value of the difference between $||{\bf v^{\leq i}}||_2$ and
	$||{\bf v^{\leq i-1}}||_2$ is at most 2.
	The claim easily follows since ${\bf x^{\leq i}} = ||{\bf v^{\leq i}}|| + \frac{ci}{2\sqrt{T}}$.
\end{proof}

The above two Claims together with Azuma's inequality gives: 
$$Pr[|{\bf x^{\leq T'}} - {\bf x^{\leq 0}}| \geq \Delta] \leq e^{-\Delta^2/2c^2T}.$$

%\begin{lemma}
%\label{concentration-inequality}
%Let ${\ D}$ be a distribution on $\R^n$ such that
%$\Exp_{v \in {\ D}} [ \parallel v \parallel _2] =0$
%and for all $v \in {\ D}$, $\parallel v \parallel _2 \leq 1$.
%Let $V_1,\ldots, V_T$ be random samples from ${\ D}^T$.
%and let ${\bf v^{\leq T}} = \sum_{i=1}^T {\bf v_i}$ be the sum of the ${\bf v_i}$'s.
%For all $\Delta>0$,
%% $0 \leq \gamma \leq T-{\sqrt T}$,
%$$\Pr_{{\bf v}} [ || {\bf v^{\leq T}} || _2 \geq 4\sqrt{T} + \Delta] \leq e^{-\Delta ^2/2T}.$$
%\end{lemma}

\begin{proof}(of Lemma~\ref{lem:concentration-inequality})
	
	In order for ${\bf v^{\leq T}}$ to have length at least
	${\sqrt T}(1+c/2) + \Delta$, there must be some largest time $t \in [T]$
	such that $||{\bf v^{\leq t}}||_2 \in (\sqrt{T}, \sqrt{T}+1]$.
	That is, at all times $t' \geq t$ the vector ${\bf v^{\leq t'}}$ is outside the
	ball of radius ${\sqrt T}$.
	Thus by the above argument, the random variables
	${\bf x^{\leq i}}_{i=t} ^T$ are a supermartingale where
	the absolute value of the difference between successive variables is at most $c$, and by Azuma,
	$\Pr[{\bf x^{\leq T}} \geq \sqrt{T} + \Delta]$
	is at most $e^{-\Delta^2/2c^2T}$. 
	Since ${\bf x^{\leq T}} = ||{\bf v^{\leq T}}||_2 - \frac{Tc}{2\sqrt{T}} =
	||{\bf v^{\leq T}}||_2 - \frac{\sqrt{T}c}{2}$,
	$\Pr[||{\bf v^{\leq T}}||_2 \geq \sqrt{T}(1 + c/2) + \Delta]$
	is at most $e^{-\Delta^2/2c^2T}$.
	
	%We consider the length of the sequence of vectors
	%$\{{\bf v^{\leq i}} = \sum_{i=1}^T {\bf v_i} \}_{i=0,\ldots,T}$, where
	%${\bf v^{\leq 0}}$ is the $0$-vector, and ${\bf v^{\leq T}}$ is the sum of
	%all $T$ vectors.
	%In order for ${\bf v^{\leq T}}$ to have length at least
	%$4\sqrt{T} + \Delta$,  there must be some largest $t \in [T]$
	%such that
	%$\parallel W^t \parallel _2 \leq \sqrt{T}$.
	%Since the length can increase/decrease by at most one at each
	%step, $\parallel W^t \parallel _2 \geq \sqrt{T}+1$.
	%Thus we consider the sequence of random variables 
	%$\overline{X_t}, \ldots, \overline{X_T}$ corresponding to the
	%length of the vectors $\overline W^t,\ldots,W^T$ (but
	%stopping early if we reach a vector of length below $4\sqrt{T}$).
	
	%By the above two Claims together with Azuma's inequality
	%(with $v_0=W^t$, $V_i = W^{t+i}$, $V_N=W^{T}$) we have: 
	
	%$$Pr[|{\overline X_T} - {\overline X_t}| \geq 4\sqrt{T} + \delta] \leq e^{(-\delta ^2/2T)}.$$
	%
	%Since ${\overline{X_t}} = \sqrt{T}$ and ${\overline{X_T}}= \parallel W^T \parallel _2 - \frac{T}{3\sqrt{T}} = \parallel W^T \parallel_2 - $ 
\end{proof}

\end{document}